\documentclass[sigconf,nonacm]{acmart}

\setcopyright{none}

\usepackage{graphics}
\usepackage{graphicx}
\usepackage{adjustbox}
\usepackage{xspace}
\usepackage{subcaption}
\usepackage{amsmath}
\usepackage{amsthm}

\usepackage{mathrsfs}
\usepackage{array}
\usepackage{textcomp}
\usepackage{weiwAlgorithm}
\usepackage{url}
\usepackage{diagbox}
\usepackage{lineno}
\usepackage{color}
\usepackage{xcolor}
\usepackage{colortbl}
\usepackage{booktabs}
\usepackage{listings}
\usepackage{tcolorbox}
\tcbuselibrary{breakable}
\usepackage{setspace}
\usepackage{multirow}
\usepackage{tabularx}
\usepackage{makecell}
\usepackage{enumitem}
\usepackage{float}
\usepackage{tikz}
\usetikzlibrary{arrows.meta,positioning,shapes.geometric,fit,calc,backgrounds}

\usepackage[normalem]{ulem}
\useunder{\uline}{\ul}{}

\definecolor{kgblue}{RGB}{55,115,179}
\definecolor{mainresultblue}{HTML}{DAE8FC}
\definecolor{skillgreen}{RGB}{46,153,83}
\definecolor{evidenceorange}{RGB}{220,130,30}
\definecolor{softgray}{RGB}{245,247,250}
\definecolor{darkgray}{RGB}{55,60,68}
\definecolor{warnred}{RGB}{196,77,88}
\definecolor{purplemem}{RGB}{126,87,194}

\newtheorem{proposition}{Proposition}

\newtheorem{definition}{Definition}

\newcommand{\method}{SkillZip\xspace}
\newcommand{\alfworld}{ALFWorld\xspace}
\newcommand{\AlgBadge}[2][black]{%
  \begingroup
  \setlength{\fboxsep}{1.3pt}%
  \fcolorbox{#1}{white}{\textcolor{#1}{\scriptsize\sffamily\bfseries #2}}%
  \endgroup
}
\newcommand{\SecTool}[1]{\ensuremath{\mathop{\text{\AlgBadge[kgblue]{#1}}}\nolimits}}
\newcommand{\ZipTool}[1]{\ensuremath{\mathop{\text{\AlgBadge[evidenceorange]{#1}}}\nolimits}}
\newcommand{\HydTool}[1]{\ensuremath{\mathop{\text{\AlgBadge[purplemem]{#1}}}\nolimits}}
\newcommand{\UpdTool}[1]{\ensuremath{\mathop{\text{\AlgBadge[warnred]{#1}}}\nolimits}}
\newcommand{\ChkTool}[1]{\ensuremath{\mathop{\text{\AlgBadge[skillgreen]{#1}}}\nolimits}}
\newcommand{\EvalTool}[1]{\ensuremath{\mathop{\text{\AlgBadge[darkgray]{#1}}}\nolimits}}
\newcommand{\myparagraph}[1]{\vspace{0.5mm}\noindent\textbf{#1}.}

\newcommand{\myparagraphunderline}[1]{\vspace{0.5mm}\noindent\underline{#1.}}
\newcommand{\myparagraphquestion}[1]{\vspace{0.5mm}\noindent\textbf{#1?}\enspace}
\newcommand{\myparagraphunderlinenew}[1]{\vspace{0.5mm}\noindent\underline{#1,}}

\newcommand{\eg}{{e.g.,}\xspace}

\AtBeginDocument{%
  }

\title{SkillZip: Contract-Preserving Graph Compression
for Scalable Agent Skill Libraries}

\author{Xingyu Tan}
\orcid{0009-0000-7232-7051}
\affiliation{%
  \institution{UNSW \& CSIRO}
  \city{Sydney}
  \country{Australia}}
\email{xingyu.tan@unsw.edu.au}

\author{Xiaoyang Wang}
\orcid{0000-0003-3554-3219}
\affiliation{%
  \institution{UNSW}
  \city{Sydney}
  \country{Australia}}
\email{xiaoyang.wang1@unsw.edu.au}

\author{Qing Liu}
\orcid{0000-0001-7895-9551}
\affiliation{%
  \institution{CSIRO}
  \city{Hobart}
  \country{Australia}}
\email{q.liu@csiro.au}

\author{Xiwei Xu}
\orcid{0000-0002-2273-1862}
\affiliation{%
  \institution{CSIRO}
  \city{Sydney}
  \country{Australia}}
\email{xiwei.xu@csiro.au}

\author{Xin Yuan}
\orcid{0000-0002-9167-1613}
\affiliation{%
  \institution{CSIRO \& UNSW}
  \city{Sydney}
  \country{Australia}}
\email{xin.yuan@csiro.au}

\author{Liming Zhu}
\orcid{0000-0001-5839-3765}
\affiliation{%
  \institution{CSIRO}
  \city{Sydney}
  \country{Australia}}
\email{liming.zhu@csiro.au}

\author{Wenjie Zhang}
\orcid{0000-0001-6572-2600}
\affiliation{%
  \institution{UNSW}
  \city{Sydney}
  \country{Australia}}
\email{wenjie.zhang@unsw.edu.au}

\renewcommand{\shortauthors}{Xingyu Tan et al.}

\begin{document}

\begin{abstract}
Large Language Models (LLMs) increasingly act as agents whose procedural knowledge is stored in reusable skill packages and loaded at inference time. As skill libraries grow, a central challenge is to expose the smallest sufficient executable context under a limited context budget. Existing systems struggle to reuse routines below the whole-skill level, preserve procedural contracts during compression, keep compressed routines executable and expandable, and update the compressed library as skills evolve.  These challenges reveal a unit mismatch: skills are retrieved as packages, compressed as text, and converted into execution graphs only after retrieval, whereas reliable reuse requires a contract-bearing procedural unit.  We propose \textbf{SkillZip}, an execution-aware procedural abstraction framework that performs contract-preserving compression over section-level graphs. SkillZip rewrites recurring contract-valid motifs into reversible ported macros while preserving boundary signatures, dependency closure, verifier reachability, and source-level expansion. At inference time, it hydrates a compact, dependency-closed context and expands macros only when required. ReZip further integrates new skills and revises risky macros using execution evidence. Comprehensive experiments on technical and embodied agent benchmarks show SkillZip consistently outperforms the strongest baseline by up to 12.2 points, while achieving a 3.46$\times$ compression ratio with 99.2\% dependency preservation and 98.7\% verifier reachability. Scaling analyses further confirm robust retrieval across skill libraries ranging from 200 to 100K skills.
\end{abstract}


\maketitle

\section{Introduction}
\label{sec:intro}
\vspace{-1mm}

\begin{figure}[t]
    \centering
    \includegraphics[width=0.99\columnwidth]{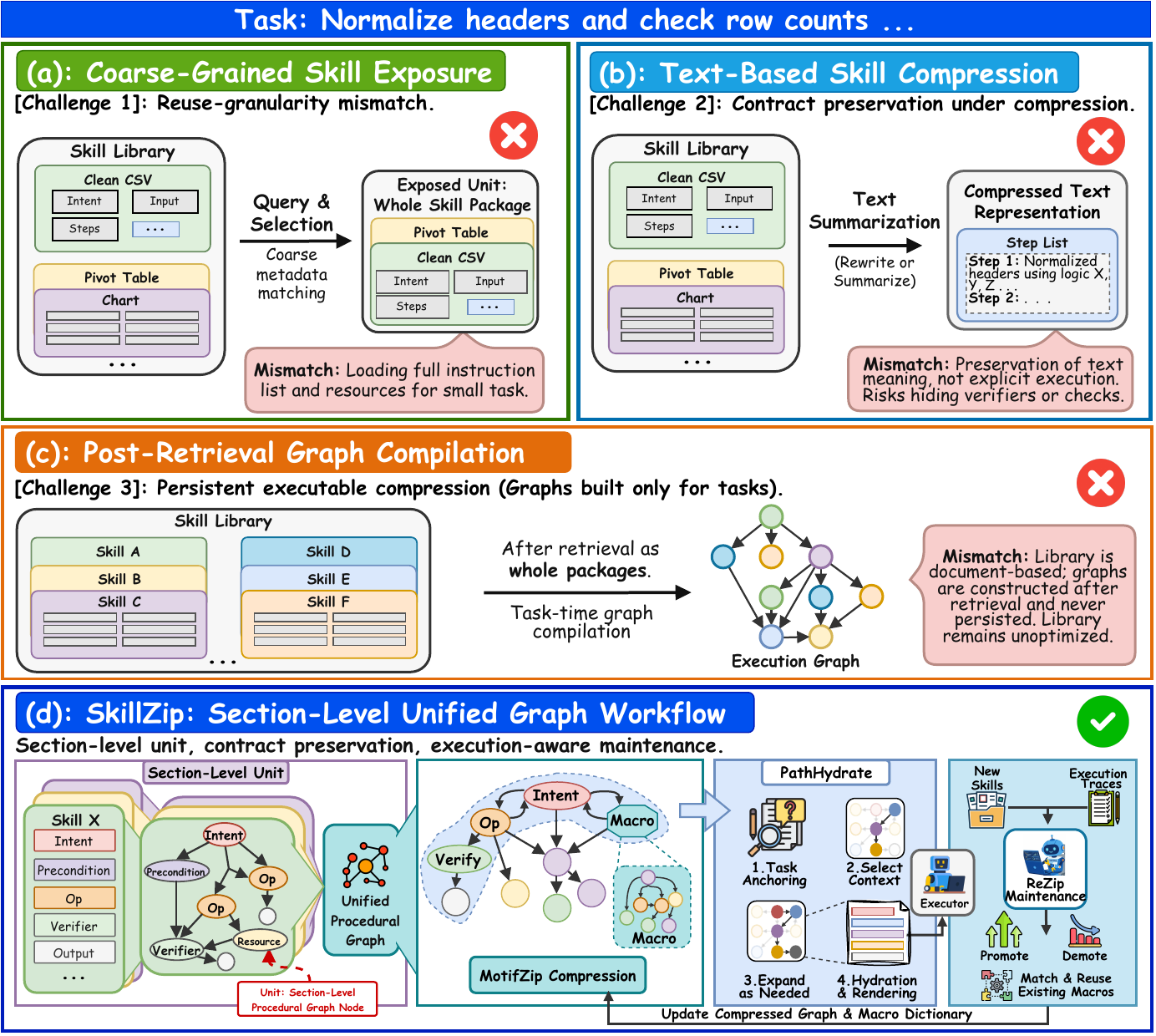}
    \vspace{-3mm}
    \caption{Representative skill-library workflows.}
    \label{fig:baseline_workflows}
    \vspace{-2mm}
\end{figure}

Large Language Models (LLMs) are increasingly used as agents that interact with
tools and environments \cite{react, toolllm}. In these settings, task success
often depends on procedural rather than factual knowledge, such as how to normalize a spreadsheet or verify that an artifact is
correct. Agent skills have
therefore emerged as an external procedural memory layer for LLM agents
\cite{agentskills, skillsbench}. A skill package stores instructions,
supporting resources, and rules for execution and verification in an editable
artifact loaded at inference time \cite{agentskills}, allowing procedures to
change without retraining. At inference time, a skill provider selects
task-relevant procedures from the library and exposes them to the agent.
As the library grows, the provider must distinguish among more overlapping
packages while avoiding incomplete or redundant context. This raises a central
question. How can a skill provider retrieve and expose the smallest sufficient
executable context within a tight context budget?

Existing solutions retrieve each skill package as a whole. However, a skill is
not an atomic text unit, but a collection of functional sections such as
intents, operations, verifiers, and outputs.
A system should therefore retrieve the
smallest execution-complete subset rather than read every relevant skill from
beginning to end. Together with their typed dependencies and verifier hooks,
they define the skill's \textbf{procedural contract}: what the
procedure exposes, how it executes, and how its effects are verified.
This issue becomes more pronounced when a task requires multiple skills, since
their useful procedures may overlap. For example, skills of \emph{Clean CSV} and \emph{Pivot Tables} may share schema inference, header
normalization, and row-count validation, while differing in their downstream
operations and verifiers. Instead of repeatedly loading each full package, the
system should retrieve the shared contract-compatible procedure once and add
only the task-specific sections.

\vspace{1mm}
\myparagraph{Challenges in existing methods}
Most existing skill-library systems can be viewed as following a
``retrieve-compress-execute'' pipeline. In this paradigm, an agent
first retrieves relevant skill packages or metadata, optionally
compresses the selected skill content to reduce prompt cost, and
then builds a task-time execution context or graph for the selected
skills. While this improves modularity and token efficiency, several
challenges remain.

\vspace{0.3mm}
\myparagraphunderline{Challenge 1: Reuse-granularity mismatch}
Existing progressive-disclosure and skill-graph systems \cite{agentskills,graphofskills,
groupofskills}, as shown in Figure~\ref{fig:baseline_workflows}(a),
use the whole skill as the retrieval unit. They reduce the number of packages
loaded, but cannot select only the relevant procedure within each package. For
example, a query to normalize headers and verify row counts
matches both the \emph{Clean CSV} and
\emph{Pivot Table} skills because they contain the same routine. A skill-level
retriever therefore loads both packages, although the task needs only their
shared sections. Retrieving this routine at the section level instead avoids
unrelated downstream operations, reducing both excess context and ambiguity
between overlapping skills. This ambiguity compounds as the library grows,
because every overlapping package enters retrieval as another coarse candidate.

\myparagraphunderline{Challenge 2: Contract preservation under compression}
Current skill-compression methods \cite{skillreducer,skillee,skim}, as shown in Figure~\ref{fig:baseline_workflows}(b), shorten skills by rewriting, debloating, or encoding their content into compact text or sequence representations. These methods optimize the token budget directly, but 
their preserved structure is mainly textual or latent, rather than an explicit procedural contract.
For executable skills, textual closeness does not imply contract equivalence: a compressed skill can stay close to the original wording while obscuring a precondition, guard branch, or verifier hook. For example, the two relevant skills above contain near-identical loading routines that feed different verifiers;
compressing them as similar spans can blur incompatible checks and make the resulting procedure unsafe to execute.

\vspace{0.3mm}
\myparagraphunderline{Challenge 3: Persistent executable compression}
Some task-time execution-graph systems~\cite{skilldag, graphofskills, agentskillos}, as shown in
Figure~\ref{fig:baseline_workflows}(c), organize already selected skills into
explicit procedural graphs and support verification or repair during execution.
However, these graphs are built after retrieval, over skills that have already
been selected, so the library itself is not stored as a persistent compressed
structure. 
As a result, the
recurring loading-and-validation routine is re-discovered and re-verified on
every task that touches it, rather than compressed into the units later trusted
by execution.

\vspace{0.3mm}
\myparagraphunderline{Challenge 4: Execution-aware maintenance}
Skill libraries are not static because new skills reveal recurring procedures, while
execution traces reveal which abstractions are reliable or risky
\cite{skillpro,reasoningbank}. A one-shot compressor can neither recognize a
routine that becomes reusable only after later skills arrive nor revise a macro
that repeatedly triggers expansion, verifier failure, or downstream repair.
For example, later spreadsheet skills may establish period alignment and balance
checking as a reusable routine, whereas formula-bearing tasks may expose a
generic export macro that lacks a required verifier. 
Compression should therefore be maintained as the library
grows and as execution evidence accumulates.

\vspace{0.3mm}
These challenges reveal a common unit mismatch. Current skill provider systems retrieve whole
skill packages, compress skill text, and build execution graphs only after
retrieval, i.e., three decisions made over three different units. For agent skills,
they should target the same object: a contract-bearing section subgraph
whose procedural contract survives retrieval, compression, and execution alike.
We call a compression
\textbf{contract-preserving} if it shortens a routine while retaining the
interface, execution, and verification aspects of this contract. Source pointers
are preserved separately so that the routine can be expanded back to its
original sections when needed.

\vspace{1mm}
\myparagraph{Contribution}
In this paper, we introduce \textbf{\method}, 
a contract-preserving graph
compression framework for scalable agent skill libraries
as shown in Figure~\ref{fig:baseline_workflows}(d). \method
changes the basic representation unit from whole skill packages to
source-grounded, contract-bearing section-level unit connected by procedural dependencies.
Over this unified representation, \method performs \textbf{execution-aware procedural
abstraction} through contract-preserving compression, i.e., recurring subgraphs are
rewritten as reusable macros only when their procedural contracts remain explicit
and recoverable. The resulting graph provides a common basis for exposing reusable
routines and hydrating task-specific executable context.

\myparagraphunderlinenew{To address reuse-granularity mismatch}
\method performs Sec2-Graph, 
which opens each skill package into source-grounded 
section nodes. It represents functional sections with distinct execution roles (\eg intents, inputs, operations, verifiers, and output) as reusable procedural units, making internal skill components visible without manually splitting the skill library.

\vspace{0.3mm}
\myparagraphunderlinenew{To preserve procedural contracts under compression}
SkillZip introduces
MotifZip, a contract-preserving compressor that mines recurring section-level
motifs and promotes them into ported macro nodes only when their boundary
signatures, dependency closure, verifier reachability, and source expansion are
preserved. Each macro is therefore a reversible rewrite instead of a lossy
summary.

\vspace{0.3mm}
\myparagraphunderlinenew{To make compressed routines executable and expandable}
\method employs  PathHydrate, which routes over the compressed graph at query
time. It constructs a compact dependency-closed procedural subgraph, repairs
missing execution roles when necessary, and renders each macro at the lowest
sufficient level, including name, contract, outline, or full source.

\vspace{0.3mm}
\myparagraphunderlinenew{To maintain compression as the library evolves}
\method introduces ReZip, which incrementally updates the compressed graph as
new skills and execution traces arrive. It reuses existing macros for compatible
regions, promotes recurring contract-valid residuals into new macros, and uses
execution evidence to increase hydration detail, split, or retire risky macros.
In summary, the advantages of \method can be abbreviated as
follows:

\begin{itemize}[leftmargin=*]
\item \textbf{Section-level procedural memory.}
\method organizes source-grounded, contract-bearing sections as reusable units
of procedural memory and represents each skill as an executable graph over
them. This representation exposes internal reuse while retaining execution
roles, dependencies, verifiers, and provenance.

\item \textbf{Contract-preserving macro compression.}
\method replaces recurring procedural motifs with ported macro nodes while
preserving boundary contracts, dependency closure, verifier reachability, and
reversible source expansion.

\item \textbf{Budgeted executable context hydration.}
\method retrieves compact, dependency-closed procedural subgraphs from the
compressed library and progressively expands macros only when execution
requires more detail.

\item \textbf{Execution-aware incremental maintenance.}
\method incrementally updates its macro dictionary through new-skill matching,
residual motif promotion, and risky macro revision, keeping compressed
procedures aligned with execution evidence.

\item \textbf{Effectiveness, efficiency, and scalability.}
\textbf{(a)} \method operates as a plug-and-play procedural-memory layer across six LLM
backbones on both technical and embodied benchmarks without backbone-specific
fine-tuning.
\textbf{(b)} \method achieves a 3.46$\times$ compression ratio and a 71.0\% reduction in
active storage while retaining 99.2\% dependency preservation and 98.7\%
verifier reachability. Its retrieval advantage remains as the library scales
from 200 to 100K skills.
\textbf{(c)} \method achieves the best end-task performance in every directly comparable
setting, outperforming the
strongest baseline \textsc{SkillDAG}  by up to 12.2 points on \alfworld.
\end{itemize}

\vspace{-3mm}

\section{Related Work}
\label{sec:related}
\vspace{-1mm}

\myparagraph{Agent skills and procedural memory}
Tool-augmented agents combine LLM reasoning with external actions through
prompting, learned tool invocation, and large API collections
\cite{react,toolformer,apibank,toolllm}. Beyond individual calls, reusable
procedural knowledge allows agents to accumulate executable programs, distill
feedback into experience, or retrieve workflows from prior trajectories
\cite{voyager,reflexion,expel,awm}. Agent Skills~\cite{agentskills}
formalizes this idea as deployable packages containing instructions, scripts,
references, and resources. Recent studies organize the broader skill lifecycle
through a unified taxonomy~\cite{agentskillssurvey}, characterize redundancy
and safety properties in real-world skill ecosystems~\cite{claudeskillsanalysis},
and acquire reusable skills from web interaction or heterogeneous scientific
resources~\cite{skillweaver,skillfoundry}. Systems and benchmarks further study
skill creation, management, retrieval, generation, and compatibility at scale
\cite{agentskillos,skillnet,skillsbench,skillret,sra,skillgenbench,sweskillsbench}.
Together, these studies establish skills as durable procedural memory, but they
typically reuse whole skill packages and leave recurring internal procedures
across skills under-modeled.

\myparagraph{Skill retrieval and execution}
A major challenge in large skill libraries is retrieving useful skills while
keeping the execution context coherent. Prior work addresses this problem by
progressively loading skill metadata before full bodies~\cite{agentskills},
organizing skills through dependencies, groups, conflicts, or specializations
\cite{graphofskills,groupofskills,skilldag}, and adapting retrieval across
different levels of skill granularity~\cite{skillrae,skilllens}. Other systems
construct execution graphs after skill selection~\cite{grasp,agentskillos}.
SkillGraph~\cite{skillgraph} retrieves ordered subgraphs from an evolving
skill-level dependency graph, while SkillOps~\cite{skillops} combines typed
skill contracts with an ecosystem graph for library diagnosis and maintenance.
Related graph-based retrieval and routing methods use topology to
preserve multi-step context
\cite{whengraphsrag,graphplanner,hiprag}. These methods improve
which skills are selected and how selected skills are coordinated. However,
their main focus remains skill-level retrieval or task-time orchestration.

\myparagraph{Skill and graph compression}
Prompt and skill compression methods shorten contexts through token selection,
rewriting, or compact representations
\cite{llmlingua,longllmlingua,llmlingua2,skillreducer,skillee,skim}.
Procedural-memory methods further abstract successful experience into reusable
skills, rules, or memories
\cite{ecs,skillpro,reasoningbank,mem1,memgen,after}. In parallel, graph mining
and summarization provide tools for discovering and compressing repeated
structure, including frequent-subgraph mining, MDL-based summaries,
grammar-based replacement, and incremental maintenance
\cite{subdue1994,gspan2002,gaston2004,snap2008,grass2010,vog2014,sweg2019,ssumm2020,slugger2022,graphgrammar2018,mosso2020}.
These techniques provide useful foundations, but agent skill libraries require
compression that preserves execution interfaces, dependency closure, verifier
reachability, and source provenance. 

More detailed related work is discussed in
Appendix~\ref{appendix:detailed_related_work}.

\vspace{-2mm}
\section{Preliminaries}
\label{sec:prelim}
\vspace{-1mm}

Consider an agent skill library $\mathcal{S}=\{s_1,\ldots,s_n\}$.
Each package $s\in\mathcal{S}$ contains procedural artifacts such as
instructions, scripts, tests, and resources. A conventional retriever maps a
query $q$ to packages $R(q)\subseteq\mathcal{S}$. We retain these packages as
source provenance but model the library at a finer procedural granularity.

\begin{figure*}[t]
    \centering
    \includegraphics[width=0.99\textwidth]{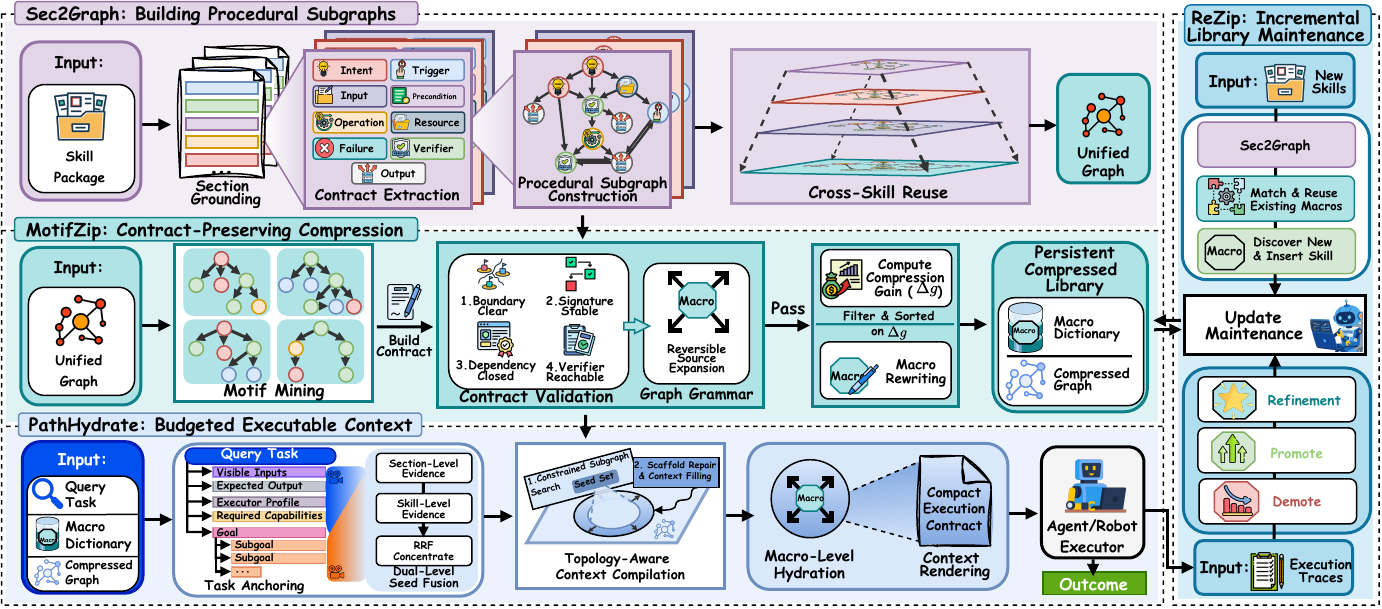}
    \caption{Overview of the \method framework. Sec2Graph retains occurrence-specific sections
    and links compatible ones through canonical prototypes; MotifZip rewrites
    recurring contract-valid subgraphs as reversible macros; PathHydrate
    compiles a budgeted executable context; and ReZip updates the compressed
    library from new skills and execution feedback.}
    \label{fig:method}
    \vspace{-3mm}
\end{figure*}


\begin{definition}[Section node]
A section node is a typed procedural unit
$v=\langle \tau_v,c_v,X_v,Y_v,R_v,G_v,\mathrm{src}_v\rangle$, where
$\tau_v$ is the execution role, $c_v$ is the section content, $X_v$ and $Y_v$
are input and output signatures, $R_v$ records resources or tools, $G_v$
records guard or verifier conditions, and $\mathrm{src}_v$ points to the
original skill source.
\end{definition}

We use nine operational roles: \textsc{Intent}, \textsc{Trigger},
\textsc{Input}, \textsc{Precondition}, \textsc{Operation},
\textsc{Resource}, \textsc{Failure}, \textsc{Verifier}, and
\textsc{Output}. These roles describe a section's procedural function rather
than its position in a source file. Each source occurrence remains a distinct
section node. Contract-compatible occurrences may additionally link to a
shared canonical prototype, which exposes reuse without erasing skill
membership, multiplicity, or occurrence-specific source pointers.
A prototype is an auxiliary node that stores a normalized role and contract
signature plus links to its source occurrences; it is not itself a source
occurrence.

\begin{definition}[Procedural skill graph]
Given $\mathcal{S}$, a procedural skill graph is
$
\mathcal{G}=(\mathcal{V},\mathcal{E}_{dep},\mathcal{E}_{skill},
\mathcal{E}_{res},\mathcal{E}_{eq}),
$
where $\mathcal{V}=\mathcal{V}_{occ}\cup\mathcal{V}_{proto}$ contains
occurrence-specific section nodes and auxiliary prototypes,
$\mathcal{E}_{dep}$ stores typed procedural dependencies,
$\mathcal{E}_{skill}$ stores skill membership edges, $\mathcal{E}_{res}$ links
sections to external resources, and
$\mathcal{E}_{eq}$ links source occurrences to compatible canonical
prototypes. We use
$\mathcal{E}=\mathcal{E}_{dep}\cup\mathcal{E}_{skill}\cup
\mathcal{E}_{res}\cup\mathcal{E}_{eq}$ to denote the full edge set.
\end{definition}

A dependency edge is $e=(u,v,\rho_e)$, where $u,v\in\mathcal{V}$ and $\rho_e$
is a typed relation such as \textsc{Requires}, \textsc{Binds},
\textsc{UsesResource}, \textsc{Verifies}, or \textsc{Repairs}. Multiple typed
edges express multi-input and multi-output procedures, while subgraph ports
expose their external contract.

The fields $(X_v,Y_v,R_v,G_v)$, typed edges, and ports together form a
procedural contract with three aspects: \emph{interface} covers typed I/O and
resource bindings; \emph{execution} covers preconditions, dependencies, guards,
effects, and failure handling; and \emph{verification} covers success conditions
and verifier hooks. The aspects are validated separately, while
$\mathrm{src}_v$ retains provenance for reversible expansion.

\begin{definition}[Skill procedural subgraph]
A skill $s$ is represented as an executable procedural subgraph
$h_s=(V_s,E_s,r_s,t_s)$ inside $\mathcal{G}$, where
$V_s\subseteq\mathcal{V}_{occ}$, $E_s\subseteq\mathcal{E}$, $r_s$ is an intent or
trigger root, and $t_s$ is a verified output section. The subgraph connects all
procedural roles required from the root to the verified output.
\end{definition}

\begin{definition}[Macro node]
A ported macro node $M_g:I_g\Rightarrow O_g$ compresses a connected section
subgraph $g=(V_g,E_g)$ with boundary inputs $I_g$, boundary outputs
$O_g$, and expansion rule
$
M_g\Rightarrow (V_g,E_g,\chi_g),
$
where $\chi_g$ is the procedural contract of the macro, retaining the
interface, execution, and verification aspects of $g$. Source pointers attached
to the stored subgraph support reversible expansion.
\end{definition}

For an operation or macro node $u$ and verifier node $z$ in a procedural
subgraph $P=(V_P,E_P)$, $z$ is reachable from $u$ if a directed path
$u\leadsto z$ exists in $(V_P,E_P\cap\mathcal{E}_{dep})$. Verifier reachability
holds when every state-changing operation or macro node in $P$ has at least one
such verifier. A valid macro preserves this condition and all three contract
aspects across compatible occurrences.

\myparagraph{Problem statement}
Given a skill library $\mathcal{S}$, \method builds a raw section graph
$\mathcal{G}$ and a macro dictionary $\mathcal{M}$, producing a compressed graph
$\mathcal{G}_{zip}=\mathcal{G}/\mathcal{M}$. Given a task query $q$, executor
profile $p$, and context budget $B$, \method as the skill provider returns a rendered context
$C_q$ to the executor by hydrating $P_q\subseteq\mathcal{G}_{zip}$. A valid $P_q$ must cover
query anchors, close required dependencies, keep verifiers reachable, and
remain expandable to source sections when needed.

\section{Method}
\label{sec:method}
\vspace{-1mm}

\method is built around one central requirement: a compressed skill library must
remain executable. A shorter context is not enough if it drops an input
contract, hides a guard, or separates an operation from its verifier. We
therefore keep the procedural contract explicit throughout the pipeline. 
As
shown in Figure~\ref{fig:method}, 
\method first opens each package into
source-grounded typed procedural sections, then compresses repeated section
subgraphs into reversible macros, and finally hydrates only the parts needed
by the current task. When the library evolves, ReZip uses execution evidence
to promote newly reusable routines and revise macros that become unsafe.



\vspace{-3mm}
\subsection{Sec2Graph: Building Procedural Subgraphs}
\label{sec:method:sec2graph}


Sec2Graph addresses the first obstacle in skill compression: a whole skill
package is too coarse to reveal reusable or execution-critical parts. Since a
package may mix instructions, scripts, schemas, tests, examples, and warnings,
Sec2Graph converts it into a source-grounded procedural subgraph that MotifZip
can compress and PathHydrate can query. The pseudocode is shown in
Appendix~\ref{appendix:alg:sec2graph}.

\myparagraph{Section node construction}
We first convert package content into typed section nodes with explicit sources
and procedural contracts.


\myparagraphunderline{Section grounding}
Sec2Graph grounds a skill package by splitting it into source-traceable
sections and assigning each section an execution role. Given a skill package
$s$, it first segments candidate sections
$\mathcal{B}_s=\textsf{SegmentSkill}(s)$ using headings, lists, code blocks,
warnings, argument descriptions, tool references, and tests as boundary cues.
When these cues are incomplete, a model-assisted parser refines the boundaries
while preserving source pointers. It then assigns each candidate
$b\in\mathcal{B}_s$ an execution role $\tau_b=\textsf{InferRole}(b)$, producing
role-labeled units for later contract extraction.
For example, in a tabular skill, ``infer the delimiter'' is an \textsc{Operation}, while ``verify that row counts are unchanged'' is a \textsc{Verifier}.


\myparagraphunderline{Contract extraction}
Typing alone is insufficient because similar operation sections may require
different inputs or checks. Sec2Graph extracts a compact local contract for each
section node $v$, including typed I/O $(X_v,Y_v)$ and resources $R_v$ for the
interface, while guards and verifier conditions are recorded in $G_v$.
Dependency, repair, and verifier edges make execution and validation relationships explicit, and the source pointer $\mathrm{src}_v$ keeps the original package traceable.


\myparagraph{Procedural subgraph construction}
After section nodes are grounded, Sec2Graph connects them into an executable
procedural subgraph. Weak-order edges preserve local order inside a skill;
dependency edges bind operations to inputs, preconditions, and resources;
verifier edges connect operations to reachable checks; and repair edges connect
failure handlers to the guards that trigger them. A skill membership edge
records which nodes belong to the same source package. For each skill, Sec2Graph
marks an intent or trigger node as the root and a verified output node as the
terminal node. This turns the package from a text bundle into a structured
procedure with explicit start, requirements, actions, and validation.

\myparagraph{Cross-skill reuse}
Finally, Sec2Graph exposes reuse across the library. Compatible sections remain
occurrence-specific nodes and are linked to a shared canonical prototype only
when their roles, boundary signatures, resources, and verifier behavior agree.
The prototype makes common routines visible, such as tabular ingestion shared
by cleaning, pivoting, and plotting skills, while each occurrence keeps its
skill membership, local edges, and source pointer. Motif support is counted
over source occurrences rather than prototypes. The output is the persistent
section graph $\mathcal{G}$, which supports structural compression without
losing occurrence identity.

\vspace{-3mm}
\subsection{MotifZip: Contract-Preserving Compression}
\label{sec:method:motifzip}
Once Sec2Graph exposes the library as procedural subgraphs, repeated routines
become visible. The difficulty is deciding which repetitions are safe to
compress. Surface similarity
alone is unsafe, i.e., two spans may both describe
``load data'' while requiring
different outputs or verifiers.
MotifZip therefore uses a lightweight typed
graph grammar to mine recurring typed subgraphs and replace only
contract-valid ones with reversible macros. Each macro records typed ports,
executable contracts, verifier hooks, and source expansion pointers. The
pseudocode is shown in Appendix~\ref{appendix:alg:motifzip}.

\myparagraph{Interface-aware motif mining}
MotifZip avoids unrestricted frequent subgraph mining over the whole library.
It first proposes candidates from compatible execution interfaces, and then
counts only occurrences whose boundary behavior is consistent.

\myparagraphunderline{Typed candidate generation}
Because skill procedures contain typed roles, directed dependencies, and
weak-order edges, candidate search starts from execution interfaces. Sections
are grouped by role signature, resource family, and I/O shape, and candidates
are grown along dependency and weak-order edges:
\begin{equation}
\mathcal{C}=\textsf{GrowMotifs}(\textsf{BucketBySignature}(\mathcal{G}),\mathcal{G}).
\label{eq:motif_candidates}
\end{equation}
Prototype links provide the grouping index, but motif growth and occurrence
matching are performed on the occurrence-specific procedural graph. Semantic
similarity is used only after this structural compatibility check. Each
candidate motif is represented as a typed attributed subgraph
$g=(V_g,E_g,\ell_g)$, where $\ell_g$ records role labels, I/O signatures,
resource families, and verifier tags. This typed bucketing restricts motif
growth to interface-compatible neighborhoods instead of comparing arbitrary
subgraphs across the full library.

\myparagraphunderline{Occurrence support}
A candidate motif is useful only if it reappears with compatible external
behavior. MotifZip therefore matches each $g$ back to the raw graph and records
its occurrences. An occurrence $\omega=(\psi_\omega,\phi_\omega)$ maps motif
nodes to section nodes through $\psi_\omega$ and motif ports to
boundary-crossing edges through $\phi_\omega$. Support is counted over
non-conflicting occurrences: two occurrences conflict if they share an internal
source section node or require incompatible port reconnections. Thus, support
measures reusable procedural structure rather than repeated text spans.
Canonical prototypes guide candidate matching but are not counted as
occurrences.

\myparagraph{Contract-preserving grammar construction}
After motif support is established, MotifZip checks whether a supported
subgraph can be replaced without changing how the surrounding graph calls,
executes, or verifies it. This stage validates the macro contract and records a
reversible grammar rule.

\myparagraphunderline{Contract validation}
For each supported motif, MotifZip builds a macro boundary and contract
\((I_g,O_g,\chi_g)=\textsf{BuildContract}(g,\Omega_g)\), where
$\Omega_g$ is the set of compatible occurrences. The motif is accepted only if
three conditions hold. First, its interface is stable: input/output ports,
role signatures, and resource requirements agree across occurrences. Second,
its execution is closed: dependencies are either internal to the motif or
explicitly exposed through macro ports. Third, its verification remains
reachable: every state-changing operation keeps its verifier inside macro
contract or reachable from macro output. These checks prevent compression
from hiding cross-boundary dependencies or detaching operations from their
required checks, so each macro preserves its 
recorded executable contract.

\myparagraphunderline{Graph grammar rule}
An accepted motif becomes a production rule in a typed attributed graph grammar:
\begin{equation}
M_g[I_g,O_g]\Rightarrow (V_g,E_g,\pi_g,\chi_g),
\label{eq:grammar_rule} 
\end{equation}
where $M_g$ is a nonterminal macro node, $I_g$ and $O_g$ are typed ports,
$\pi_g$ records occurrence-specific node and port mappings, and $\chi_g$ is the
executable contract. Rewriting an occurrence removes only the internal nodes of
the matched motif, inserts $M_g$, and reconnects external edges through the port
map $\phi_\omega$. Expansion performs the inverse operation and restores the
source-grounded section graph.

\myparagraph{Macro selection and representation}
The last step decides which contract-valid motifs should actually enter the
macro dictionary. MotifZip first scores the candidate by compression benefit and
execution risk, then performs conflict-aware rewriting.

\myparagraphunderline{Compression gain}
Not every valid motif is worth compressing. MotifZip scores each motif by a
local description-length gain:
\begin{equation}
\begin{aligned}
\Delta(g) =~~&
\mathrm{freq}(g)L(g)-L(M_g)-L(\mathrm{rule}_g)\\
&+\alpha\,\mathrm{Reuse}(g)-\lambda\,\mathrm{Cut}(g)
-\mu\,\mathrm{Risk}(g).
\end{aligned}
\label{eq:motif_gain}
\end{equation}
Here $\mathrm{freq}(g)$ is the number of non-conflicting occurrences selected
for rewriting. The first term rewards replacing repeated instances of $g$ with
one macro and one expansion rule. $\mathrm{Reuse}(g)$ favors reuse across
distinct skills or task families; $\mathrm{Cut}(g)$ penalizes boundary loss; and
$\mathrm{Risk}(g)$ penalizes weak verifier support or ambiguous contracts.
The nonnegative weights $\alpha$, $\lambda$, and $\mu$ control these
MotifZip-specific terms.
Thus, the score selects a macro only when its description-length saving
justifies the execution risk.

\myparagraphunderline{Macro rewriting}
Candidates are processed greedily in descending $\Delta(g)$. MotifZip skips
overlapping rewrites, creates a ported macro for each accepted motif, and stores
the production rule, occurrence mappings, boundary ports, verifier hooks, source
pointers, and rendering levels. 
A macro can be rendered as a name, contract,
outline, or full source. It is therefore a reversible procedural rewrite rule,
not an opaque summary. The conflict-aware greedy policy keeps rewriting
tractable over large libraries and does
not claim global optimality over all overlapping
motifs.

\begin{proposition}[Compositional structural lifting]
\label{prop:structural_lifting}
Let $\Omega$ be a set of pairwise non-conflicting motif occurrences accepted by
MotifZip. Denote $\kappa_{\Omega}$, $\xi_{\Omega}$ as their simultaneous
macro rewriting and source expansion. If a raw procedural subgraph $P$
contains, for each occurrence in $\Omega$, either all or none of its internal
nodes, then $\xi_{\Omega}(\kappa_{\Omega}(P))\cong P$ up to auxiliary prototype
links, and isomorphism preserves typed external dependencies and
operation-to-verifier reachability.
\end{proposition}

\begin{proof}[Proof sketch]
Consider first a single occurrence $\omega$. \emph{(i)~Identification.} The occurrence map $\pi_g$ identifies its source nodes, and $\phi_\omega$
maps every boundary-crossing edge to a typed macro port; dependency closure
guarantees that no crossing edge is left unrecorded. \emph{(ii)~Expansion.}
$\xi$ therefore restores exactly original internal nodes and edges and
reconnects every external edge through its recorded port, so each required verifier path is either untouched or recovered from the stored occurrence. \emph{(iii)~Composition.} Because MotifZip rejects
conflicting occurrences, rewriting one occurrence alters neither the internal
nodes nor the port map of any other; the single-occurrence argument thus applies independently to each element of
$\Omega$, and induction over $|\Omega|$ yields the stated isomorphism.
\end{proof}

This lifting property links compression to execution: MotifZip rewrites only
regions whose recorded source structure can be recovered compositionally, and
PathHydrate relies on this guarantee when retrieving from the compressed graph.
The guarantee is structural rather than semantic. It recovers the recorded
interfaces, dependencies, verifier paths, and source provenance, but does not
establish the equivalence of unrecorded behavior or the correctness of a
verifier itself. Section~\ref{sec:exp} and
Appendix~\ref{appendix:contract_results} then evaluate how reliably these
recorded contracts are preserved in practice.

\vspace{-2mm}
\subsection{PathHydrate: Budgeted Executable Context}
\label{sec:method:pathhydrate}
\vspace{-1mm}

MotifZip produces a compact library, but a task still needs an execution context
that is both small and complete. PathHydrate solves this online problem. It
maps the query to procedural anchors, concentrates section seeds without losing
skill-level coherence, searches for a compact connected subgraph, and then
renders macros only to the detail level required by the task. 
The pseudocode is shown in Appendix~\ref{appendix:alg:pathhydrate}.

\myparagraph{Query-guided seed construction}
This stage translates a natural-language task into graph anchors and then uses
both section-level and skill-level evidence to build a coherent seed set.

\myparagraphunderline{Task anchoring}
PathHydrate first converts the raw query into a structured task object:
\begin{equation}
z_q=(g_q,O_q,\Gamma_q,I_q,D_q,\{d_i\}),  
\end{equation}
compactly written as
$z_q=\textsf{AnalyzeTask}(q,p)$. Here $g_q$ is the goal, $O_q$ are expected
outputs, $\Gamma_q$ are required capabilities, $I_q$ are visible inputs, $D_q$
is a domain or executor profile, and $\{d_i\}$ is an ordered list of subgoals.
These anchors translate task language into the same execution-role space used by
Sec2Graph.

\myparagraphunderline{Dual-level seed fusion}
The anchors are then mapped to section seeds through two complementary views.
Let $\mathbf{e}(\cdot)$ denote the embedding function used for dense matching.
Each section node $v$ is scored by the better match under the subgoal and  raw-query views:
\begin{equation}
s(v,d_i,q)=
\max\{\cos(\mathbf{e}(d_i),\mathbf{e}(v)),\cos(\mathbf{e}(q),\mathbf{e}(v))\}.
\label{eq:seed_score}
\end{equation}
The subgoal view emphasizes capability and output-signature matching, while the
raw-query view preserves names, resources, and entities that may be omitted by
the abstraction. Because section-level retrieval can scatter seeds across
unrelated skills, PathHydrate also builds two skill-level rankings: one from
skill descriptions and one from the best section match inside each skill. For a
candidate skill $\sigma$, the rankings are fused by reciprocal-rank fusion:
\begin{equation}
\mathrm{RRF}(\sigma)=\sum_{r\in\mathcal{R}_{rank}}
\frac{1}{k+\mathrm{rank}_{r}(\sigma)},
\label{eq:rrf}
\end{equation}
where $\mathcal{R}_{rank}=\{R_{\mathrm{doc}},R_{\mathrm{node}}\}$ contains
skill-description and node-max rankings, respectively;
$\mathrm{rank}_{r}(\sigma)$ is the one-based position of $\sigma$ in ranking
$r$, and is set to $+\infty$ when $\sigma$ is absent; and $k$ is smoothing
constant.
The fused skill set concentrates the seed pool while still allowing
high-confidence section-level rescue. This step connects fine-grained section
precision with package-level coherence.

\myparagraph{Topology-aware context compilation}
Once the seed pool is coherent, PathHydrate must turn it into executable
context. It first searches for a compact connected subgraph, then repairs the
role scaffold that execution requires.

\myparagraphunderline{Constrained subgraph search}
Given the fused seeds, PathHydrate searches for a compact graph object that can
serve as executable context. The ideal objective is:
\begin{equation}
\begin{aligned}
P_q^\star
=\operatorname*{arg\,min}_{P\subseteq\mathcal{G}_{zip}}\quad
\eta T(P)+\beta |P|+\gamma E(P)
-\delta\,\mathrm{Match}(P,q),
\end{aligned}
\label{eq:path_cost}
\end{equation}
subject to anchor coverage, dependency closure, verifier reachability, and the
budget $T(P)\leq B$. Here $T(P)$ is the estimated token cost of the selected
procedural payload at its current hydration levels before rendering. $|P|$ is the number of
hydrated sections or macros, $E(P)$ estimates future expansion cost, and
$\mathrm{Match}(P,q)$ measures coverage of task anchors and matched section
seeds. The nonnegative weights $\eta$, $\beta$, $\gamma$, and $\delta$ are
specific to the PathHydrate objective.
At runtime, PathHydrate adds low-cost connectors until the required roles are
closed or the next addition would exceed the procedural-content budget. Thus, search
remains bounded by the task budget without materializing the growing raw
library.

\myparagraphunderline{Scaffold repair and context filling}
A connected subgraph can still be incomplete for execution, so PathHydrate then
repairs the scaffold around selected operations: it walks backward to recover
\textsc{Input}, \textsc{Precondition}, and \textsc{Resource} sections, and walks
forward to recover \textsc{Failure}, \textsc{Verifier}, and \textsc{Output}
sections. Closure repair produces
$(P,\xi_q)=\textsf{RepairClosure}(P,\mathcal{G}_{zip},B)$, where $\xi_q$ records
restored or infeasible roles. If no verifier-reachable context fits the budget,
the system falls back to the smallest skill-level bundle covering the anchors.
Context filling is sufficiency-gated rather than budget-seeking: it terminates
once the task anchors are covered, required dependencies are closed, and a
verifier remains reachable. 
Thus, unused budget is not spent on additional procedural text.

\myparagraph{Progressive hydration and rendering}
The selected subgraph is still an internal graph object. PathHydrate turns it
into agent-readable context by choosing macro detail levels and rendering the
selected units as an executable context.

\myparagraphunderline{Macro-level hydration}
The selected subgraph may contain either source sections or macros. For each
selected macro, PathHydrate chooses the lowest sufficient hydration level: name,
contract, outline, or full source. If the macro contract does not expose a
needed input, guard condition, verifier, or source pointer, the macro is
expanded. This upgrades progressive disclosure from package level to graph
level: \method reveals the smallest executable view of a section graph instead
of loading an entire skill body.

\newif\ifMainResultsDeltas
\MainResultsDeltastrue
\newcommand{\MainResultBestRow}{\rowcolor{mainresultblue}}
\newcommand{\MainGain}[1]{%
  \ifMainResultsDeltas\,{\scriptsize\textcolor{skillgreen}{\ensuremath{\uparrow #1}}}\fi}
\newcommand{\MainLoss}[1]{%
  \ifMainResultsDeltas\,{\scriptsize\textcolor{warnred}{\ensuremath{\downarrow #1}}}\fi}

\begin{table*}[ht]
\centering
\caption{Main results on \textsc{SkillsBench} and \alfworld. R is task reward (\%) on \textsc{SkillsBench} or episode success rate (\%) on \alfworld. Arrows report point changes from Vector
Skills. The best comparable results are in \textbf{bold}.}
\label{tab:main_results}
\vspace{-3mm}
\small
\setlength{\tabcolsep}{2.6pt}
\renewcommand{\arraystretch}{0.96}
\resizebox{0.75\textwidth}{!}{
\begin{tabular}{llcccccccc}
\toprule
\multirow{2}{*}{\textbf{Backbone}} &
\multirow{2}{*}{\textbf{Method}} &
\multicolumn{4}{c}{\textbf{\textsc{SkillsBench} \cite{skillsbench}}} &
\multicolumn{4}{c}{\textbf{\alfworld} \cite{alfworld}} \\
\cmidrule(lr){3-6}\cmidrule(lr){7-10}
& & \textbf{R$\uparrow$} & \textbf{Ret@1$\uparrow$} & \textbf{Ret@5$\uparrow$} & \textbf{MRR$\uparrow$}
  & \textbf{R$\uparrow$} & \textbf{Ret@1$\uparrow$} & \textbf{Ret@5$\uparrow$} & \textbf{MRR$\uparrow$} \\
\midrule
\multirow{5}{*}{MiniMax-M2.7}
& Vanilla Skills & 17.2\MainGain{6.8} & -- & -- & --
& 47.1\MainLoss{3.6} & -- & -- & -- \\
& Vector Skills & 10.4 & 3.6 & 10.8 & 5.8
& 50.7 & 37.9 & 68.6 & 49.2 \\
& GoS \cite{graphofskills} & 18.7\MainGain{8.3} & 50.6\MainGain{47.0}
& 65.5\MainGain{54.7} & 57.3\MainGain{51.5}
& 54.3\MainGain{3.6} & 56.4\MainGain{18.5}
& 86.4\MainGain{17.8} & 67.9\MainGain{18.7} \\
& \textsc{SkillDAG} \cite{skilldag} & 27.3\MainGain{16.9} & 66.7\MainGain{63.1}
& 78.2\MainGain{67.4} & 71.3\MainGain{65.5}
& 67.1\MainGain{16.4} & 57.9\MainGain{20.0}
& 92.1\MainGain{23.5} & 71.1\MainGain{21.9} \\
\MainResultBestRow
& \textbf{\method} & \textbf{33.3}\MainGain{22.9}
& \textbf{73.6}\MainGain{70.0}
& \textbf{92.0}\MainGain{81.2} & \textbf{81.3}\MainGain{75.5}
& \textbf{79.3}\MainGain{28.6} & \textbf{85.7}\MainGain{47.8}
& \textbf{98.6}\MainGain{30.0} & \textbf{91.2}\MainGain{42.0} \\
\midrule
\multirow{5}{*}{gpt-5.2-codex}
& Vanilla Skills & 27.4\MainGain{5.9} & -- & -- & --
& 89.3\MainLoss{3.6} & -- & -- & -- \\
& Vector Skills & 21.5 & 3.6 & 10.8 & 5.8
& 92.9 & 37.9 & 68.6 & 49.2 \\
& GoS \cite{graphofskills} & 34.4\MainGain{12.9} & 50.6\MainGain{47.0}
& 65.5\MainGain{54.7} & 57.3\MainGain{51.5}
& 93.6\MainGain{0.7} & 56.4\MainGain{18.5}
& 86.4\MainGain{17.8} & 67.9\MainGain{18.7} \\
& \textsc{SkillDAG} \cite{skilldag} & 36.8\MainGain{15.3} & 70.1\MainGain{66.5}
& 75.9\MainGain{65.1} & 73.0\MainGain{67.2}
& 93.6\MainGain{0.7} & 60.4\MainGain{22.5}
& 85.1\MainGain{16.5} & 69.6\MainGain{20.4} \\
\MainResultBestRow
& \textbf{\method} & \textbf{43.0}\MainGain{21.5}
& \textbf{74.7}\MainGain{71.1}
& \textbf{88.5}\MainGain{77.7} & \textbf{81.0}\MainGain{75.2}
& \textbf{96.4}\MainGain{3.5} & \textbf{90.7}\MainGain{52.8}
& \textbf{99.3}\MainGain{30.7} & \textbf{95.0}\MainGain{45.8} \\
\bottomrule
\end{tabular}}
\vspace{-2mm}
\end{table*}

\myparagraphunderline{Execution-context rendering}
The final context is rendered as a compact executable context rather than a
concatenation of snippets, denoted as $C_q=\textsf{RenderContract}(U,p,B)$,
where $U$ denotes the selected source sections and hydrated macro views. The
rendered context includes intent, bound inputs, required preconditions,
hydrated operations, resources, guards, failure handlers, verifiers, and source
expansion pointers. For a spreadsheet task, PathHydrate may load a shared
tabular-ingest contract plus a pivot-total verifier while leaving chart-rendering
sections out of context. The hydration log $\mathcal{L}_q$ records selected
sections, macro levels, token cost, late expansions, verifier coverage, and
source pointers. These logs connect runtime behavior back to the structural
constraints enforced by MotifZip and provide the feedback used by ReZip.


\myparagraph{Closed-loop invariants} The runtime postconditions mirror the compression constraints. Sec2Graph exposes section roles and source pointers; MotifZip accepts a macro only after boundary, signature, dependency-closure, and verifier-reachability checks; and PathHydrate records anchor coverage, restored dependencies, verifier status, macro levels, and source expansions in $\mathcal{L}_q$. The same quantities are used by the structure-aware evaluation protocol, so the empirical metrics test whether compression remains executable rather than merely shorter.

\vspace{-2mm}
\subsection{ReZip: Incremental Library Maintenance}
\label{sec:method:rezip}
\vspace{-1mm}

ReZip closes the compression-execution loop as the library evolves. We model
each update as a transition over the compressed library:
\begin{equation}
\mathcal{Z}_t=
(\mathcal{G}_{zip}^{t},\mathcal{M}_t,\mathcal{B}_{res}^{t},\Sigma_t),
\qquad
\mathcal{Z}_{t+1}=\operatorname{ReZip}(\mathcal{Z}_t,u_t),
\label{eq:rezip_state}
\end{equation}
where $\mathcal{B}_{res}^{t}$ stores unmatched residual subgraphs,
$\Sigma_t$ stores macro-level execution evidence, and $u_t$ is either a new
skill or an execution trace. New skills provide evidence for reusable
structure, while traces reveal abstractions that require revision. The
pseudocode is shown in Appendix~\ref{appendix:alg:rezip}.

\myparagraph{New-skill assimilation}
Sec2Graph first converts an arriving skill into a procedural subgraph. ReZip
matches each connected region against existing macros using typed ports and the
interface, execution, and verification aspects of their contracts. A compatible
region reuses the macro while retaining an occurrence-specific source map;
unmatched regions remain explicit and enter $\mathcal{B}_{res}^{t}$. ReZip
promotes a residual motif $r$ only when:
\begin{equation}
\operatorname{supp}_t(r)\ge m,\qquad
\Delta(r)>0,\qquad
\mathrm{Valid}_{\chi}(r)=1,
\label{eq:rezip_promote}
\end{equation}
where $\operatorname{supp}_t(r)$ is the number of distinct source skills
supporting $r$ up to update step $t$, $m$ is the minimum cross-skill support
required for promotion,
$\Delta(r)$ is the MotifZip compression gain, and
$\mathrm{Valid}_{\chi}$ applies the same port, dependency-closure, and verifier
checks as offline compression. Thus, library growth reuses established macros
immediately but introduces a new abstraction only after repeated
contract-compatible evidence.
The promoted occurrences are then removed from the residual buffer.

\myparagraph{Execution-aware macro revision}
For each observed macro $M$, $\Sigma_t(M)$ records its uses, source expansions, verifier
failures, and downstream repair cost. ReZip summarizes these signals as
\begin{equation}
\rho_t(M)=
\lambda_e\frac{n_{\mathrm{exp}}(M)}{n_{\mathrm{use}}(M)}
+\lambda_v\frac{n_{\mathrm{fail}}(M)}{n_{\mathrm{use}}(M)}
+\lambda_d\frac{c_{\mathrm{repair}}(M)}{n_{\mathrm{use}}(M)}.
\label{eq:rezip_risk}
\end{equation}
The weights $\lambda_e$, $\lambda_v$, and $\lambda_d$ balance insufficient
macro detail, failed verification, and downstream recovery effort.
When $\rho_t(M)$ exceeds a risk threshold, ReZip performs controlled demotion:
it first raises the hydration level for the affected task profile; persistent
risk causes the macro to be split into narrower contract-compatible rules or
retired in favor of its source sections. Every
promoted macro therefore satisfies MotifZip contract checks, while every
revision retains a source-grounded expansion. Consequently,
$\mathcal{G}_{zip}^{t+1}$ preserves the executable-subgraph lifting invariant
used by PathHydrate.

\vspace{-2mm}
\section{Experiments}
\label{sec:exp}
\vspace{-1mm}

In this section, we evaluate \method on \textsc{SkillsBench} and \alfworld.
The detailed experimental settings, including baselines, evaluation metrics, and implementations, are provided in
Appendix~\ref{appendix:exp_details}.

\vspace{-2mm}

\subsection{Main Results}
\vspace{-1mm}

\myparagraphquestion{(RQ1) Does \method improve end-to-end task performance}
To evaluate the overall effectiveness of \method,
we compare \method with Vanilla Skills, Vector Skills, GoS, and
\textsc{SkillDAG}, covering full-skill loading, embedding-based retrieval, and
graph-based retrieval under the same benchmark and backbone settings. As shown
in Table~\ref{tab:main_results}, \method consistently achieves the best
end-to-end performance across both benchmarks and both backbone LLMs. With
MiniMax-M2.7, \method obtains a task reward of 33.3 on
\textsc{SkillsBench} and an episode success rate of 79.3 on \alfworld,
improving by 6.0 points and 12.2 points over
\textsc{SkillDAG}. When
using gpt-5.2-codex, \method further reaches 43.0 on \textsc{SkillsBench} and
96.4 on \alfworld, improving by 6.2 points and 
2.8 points over \textsc{SkillDAG}.
Notably, both GoS and \textsc{SkillDAG} already achieve 93.6\% success on
\alfworld with gpt-5.2-codex, leaving limited room for further improvement,
yet \method still raises the success rate to 96.4\%. One possible explanation
for the weaker Vector Skills results is that whole-skill embeddings are
affected by shared boilerplate and formatting content, making closely related
packages difficult to distinguish. This interpretation is consistent with the
higher similar-skill confusion observed for skill-level retrieval in
Appendix~\ref{appendix:scale_overlap}. Overall, these consistent gains across
different backbones and task formats demonstrate retrieving
execution-complete sections provides more useful procedural context than
loading or retrieving skills as indivisible packages. On the matched MiniMax-M2.7
\textsc{SkillsBench} runs, \method also reduces cumulative prompt processing by
47.0\% and average tool calls by 21.7\% relative to \textsc{SkillDAG}, while
reducing uncached prompt input by 18.7\% and end-to-end task time by
21.1\%. Appendix~\ref{exp:end-to-end-cost} provides the
full breakdown and distinguishes these trajectory-level counters from the
one-time hydrated skill context.

\myparagraphquestion{(RQ2) Does section-level retrieval improve retrieval quality}
To assess whether the end-task gains come from more precise procedural
exposure, we evaluate retrieval on gold-annotated queries and map each retrieved section back
to its source skill. As shown in Table~\ref{tab:main_results}, \method
consistently improves retrieval quality over \textsc{SkillDAG} across both
benchmarks and backbone LLMs. On \textsc{SkillsBench}, \method increases
Ret@1 from 66.7 to 73.6 and Ret@5 from 78.2 to 92.0 with MiniMax-M2.7, while
also raising MRR by 10.0. The same trend holds with gpt-5.2-codex,
where \method improves Ret@5 by 12.6 over \textsc{SkillDAG}. The gain
is more pronounced on \alfworld: Ret@1 increases by 27.8 with
MiniMax-M2.7 and by 30.3 with gpt-5.2-codex. Since Ret@5 is already
high for strong graph-based baselines, the MRR gains indicate that \method
moves the correct source closer to the top rather than simply recovering it
somewhere in the candidate set. Overall, these results suggest that
execution-complete sections reduce skill-level retrieval ambiguity by
separating shared routines from the operations that distinguish each skill.
This benefit becomes stronger as the library grows:
Appendix~\ref{appendix:scale_overlap} shows that the Ret@1 advantage over
\textsc{SkillDAG} widens from 6.2 at 200 skills to 23.3 at 100K.
At the largest scale, \method retains 65.1 Ret@1 with 248.3 ms online
retrieval and hydration latency. After section contracts are available, local
graph construction and MotifZip complete in 178 seconds for the corresponding
100K-skill library with 4.77M section nodes
(Appendix~\ref{exp:offline-construction}).

\vspace{-2mm}
\subsection{Compression and Structural Fidelity}
\vspace{-1mm}

\begin{table}[t]
\centering
\caption{Compression and structural-fidelity results on
\textsc{SkillsBench}. The complete SkillZip configuration is shaded.}
\label{tab:compression_fidelity}
\vspace{-4mm}
\setlength{\tabcolsep}{2.2pt}
\begin{adjustbox}{max width=0.8\columnwidth}
\begin{tabular}{lcccccc}
\toprule
\textbf{Representation} & \textbf{CR$\uparrow$} & \textbf{Tok$\downarrow$}
& \textbf{DPR$\uparrow$} & \textbf{VR$\uparrow$}
& \textbf{Recover.$\downarrow$} & \textbf{R$\uparrow$} \\
\midrule
Raw section graph        & 1.00$\times$ & 6,716 & 100.0 & 100.0 & 0.0  & 31.0 \\
Exact-text dedup.        & 1.43$\times$ & 4,697 & 98.6  & 98.1  & 5.2  & 31.2 \\
Text compression~\cite{llmlingua2}
                         & 3.46$\times$ & 1,941 & 65.0  & 60.0  & 45.0 & 25.5 \\
Generic graph grammar    & 2.91$\times$ & 2,308 & 93.4  & 90.8  & 22.7 & 29.4 \\
\method w/o checks       & 3.78$\times$ & 1,777 & 88.9 & 84.6 & 31.5 & 27.8 \\
\rowcolor{mainresultblue}
\textbf{\method}         & 3.46$\times$ & 1,941 & 99.2 & 98.7 & 14.8 & 33.3 \\
\bottomrule
\end{tabular}
\end{adjustbox}
\vspace{-2mm}
\end{table}

\myparagraphquestion{(RQ3) Does compression preserve executable structure}
We evaluate whether compression can reduce the rendered skill context while
preserving the dependencies and verification conditions needed for execution.
As shown in Table~\ref{tab:compression_fidelity}, exact-text deduplication
keeps most structural information, but only provides a limited compression
ratio of 1.43$\times$, since it can merge only surface-identical sections.
Text compression reaches the same rounded compression ratio and average
rendered context as \method, but it sharply reduces structural fidelity: DPR drops to 65.0,
VR drops to 60.0, and 45.0\% of queries require recovery from the original
sections. Its reward also falls to 25.5, showing that shorter context alone
does not guarantee executable context. Generic graph grammar preserves more
graph structure than text compression, but still loses dependency and verifier
information, leading to lower reward than the raw section graph. Similarly,
removing contract checks makes \method more compact, but the loss in DPR, VR,
and reward shows that unsafe merges can damage execution.
In contrast, \method achieves a 3.46$\times$ compression ratio while keeping
DPR and VR close to the raw graph, at 99.2 and 98.7, respectively. It reduces recovery from 45.0\% to 14.8\% relative to text compression. Relative to the raw section graph, SkillZip improves reward from 31.0 to 33.3 while retaining near-complete structural fidelity.
Compared with text compression, \method improves reward by 7.8. These
results suggest that the main
benefit of \method is not only compressing repeated procedures, but compressing
them in a way that keeps their interface, execution dependencies, and verifier
contracts recoverable.
Because this preservation depends on the extracted contracts, we additionally
evaluate contract extraction against human annotations and under controlled
field corruption in Appendix~\ref{appendix:contract_results}, obtaining 91.6 macro-F1 and 84.6 exact match and maintaining high DPR and VR under 10\%  contract corruption.

\vspace{-2mm}

\subsection{Ablation Study}
\label{exp:ablation}
\vspace{-1mm}


\begin{table}[t]
\centering
\caption{Component ablation of \method on \textsc{SkillsBench}.}
\label{tab:ablation_skillzip}
\vspace{-4mm}
\setlength{\tabcolsep}{2.0pt}
\begin{adjustbox}{max width=0.9\columnwidth}
\begin{tabular}{llccccc}
\toprule
\textbf{Component} & \textbf{Variant} & \textbf{R$\uparrow$} & \textbf{Ret@1$\uparrow$}
& \textbf{Tok$\downarrow$} & \textbf{DPR$\uparrow$} & \textbf{VR$\uparrow$} \\
\midrule
\rowcolor{mainresultblue}
\textbf{Full} & \textbf{\method} & 33.3 & 73.6 & 1,941 & 99.2 & 98.7 \\
\midrule
\textbf{Unit} & w/o section-level nodes & 27.9 & 66.7 & 3,103 & -- & -- \\
\midrule
\multirow{3}{*}{\textbf{Compression}}
& w/o MotifZip            & 31.0 & 71.8 & 2,967 & 100.0 & 100.0 \\
& w/o dependency closure  & 28.6 & 72.1 & 1,653 & 82.3 & 92.6 \\
& w/o verifier constraint & 29.1 & 72.8 & 1,668 & 94.9 & 76.4 \\
\midrule
\multirow{2}{*}{\textbf{Hydration}}
& w/o global section rescue & 30.4 & 68.2 & 1,812 & 96.5 & 96.8 \\
& w/o adaptive hydration    & 31.5 & 73.2 & 2,587 & 99.2 & 98.7 \\
\bottomrule
\end{tabular}
\end{adjustbox}
\vspace{-2mm}
\end{table}

\myparagraphquestion{(RQ4) Which components contribute to the final performance}
To identify the contribution of each component, we conduct ablation studies. As shown in Table~\ref{tab:ablation_skillzip},
replacing section-level nodes with skill-level nodes causes the largest drop:
Ret@1 decreases by 6.9 points, task reward drops by 5.4 points, and rendered context
increases by 59.9\%.
This drop occurs because section-level procedures are the input to later
stages. Without them, MotifZip can only compress coarser
patterns and PathHydrate must retrieve from less precise context, limiting both
retrieval precision and context efficiency.
The compression ablations reveal a different division of labor. Without
MotifZip, DPR and VR remain perfect, but context increases by 52.9\% and reward
falls to 31.0, showing that motif abstraction primarily removes
contract-compatible repeated structure. This variant still uses PathHydrate
and therefore renders 2,967 tokens rather than the 6,716 tokens of full,
unpruned raw-graph retrieval in Table~\ref{tab:compression_fidelity}. Removing
dependency closure or verifier constraints leaves Ret@1 near the full model,
but damages executable structure: DPR falls to 82.3 without closure, while VR
falls to 76.4 without verifier constraints. 
Thus, correct retrieval alone is insufficient without execution-preserving
compression.
The hydration ablations further show how PathHydrate controls the final
context. Removing global section rescue reduces Ret@1 to 68.2 and reward to
30.4, showing that local expansion from the initial skill neighborhood can miss
useful sections. Disabling adaptive hydration preserves DPR and VR, but uses
33.3\% more tokens. This context-saving effect is further supported by the
budget analysis in Appendix~\ref{appendix:budget_results}, where PathHydrate
renders 1,941 tokens per task, 72.1\% fewer than top-5 whole-skill loading.

To further evaluate \method, we report additional studies on compression and
contract robustness, retrieval scalability, and procedural overlap across
domains
(Appendix~\ref{appendix:compression_accounting}--\ref{appendix:overlap_applicability});
hydration quality, context compactness, and system cost
(Appendix~\ref{appendix:budget_results}--\ref{appendix:cost_results});
streaming maintenance, repeated-run reliability, backbone generalization, and
failure attribution
(Appendix~\ref{appendix:rezip_results}--\ref{appendix:failure_analysis});
and case studies (Appendix~\ref{appendix:case_studies}). A detailed
outline is shown in the \hyperref[outline]{Appendix Outline}.

\vspace{-3mm}

\section{Conclusion}
\label{sec:conclusion}
\vspace{-1mm}

In this paper, we present \method, a contract-preserving graph compression
framework for scalable agent skill libraries. \method organizes skills into
section-level procedural graphs, compresses repeated execution patterns, and
builds compact task-specific contexts while preserving dependency and verifier
contracts. Experiments on \textsc{SkillsBench} and \alfworld show that
\method outperforms strong retrieval and graph-based baselines in both
end-task performance and source-skill retrieval. Overall, \method improves
skill-library scalability while maintaining executable structure, enabling
efficient and reliable skill reuse in long-horizon agent tasks.

\bibliographystyle{ACM-Reference-Format}
\bibliography{custom}

\newpage
\appendix

\onecolumn
\section*{Appendix Outline}
\label{outline}
\hypersetup{hidelinks}
\newcommand{\OutlineEntry}[4]{%
  \noindent\hspace*{#1}%
  \parbox[t]{\dimexpr\linewidth-#1\relax}{%
    \hyperref[#2]{#3}\nobreak\dotfill\nobreak\pageref{#2}}%
  \par\vspace{#4}}
\newcommand{\OutlineItem}[3]{\OutlineEntry{#1}{#2}{#3}{2pt}}
\newcommand{\OutlineItemGap}[3]{\OutlineEntry{#1}{#2}{#3}{5pt}}
\newcommand{\OutlineTop}[3]{\OutlineEntry{#1}{#2}{\textbf{#3}}{2pt}}
\newcommand{\OutlineItemCompact}[3]{\OutlineEntry{#1}{#2}{#3}{1pt}}
\newcommand{\OutlineTopCompact}[3]{\OutlineEntry{#1}{#2}{\textbf{#3}}{1pt}}
\newcommand{\OutlineItemTight}[3]{\OutlineEntry{#1}{#2}{#3}{0pt}}
\newcommand{\OutlineTopTight}[3]{\OutlineEntry{#1}{#2}{\textbf{#3}}{0pt}}
\OutlineTopTight{0pt}{appendix:alg}{A.\,Algorithm}
\OutlineItemTight{1.5em}{appendix:alg:overview}{A.1 \method Workflow}
\OutlineItemTight{1.5em}{appendix:alg:sec2graph}{A.2 Sec2Graph}
\OutlineItemTight{1.5em}{appendix:alg:motifzip}{A.3 MotifZip}
\OutlineItemTight{1.5em}{appendix:alg:pathhydrate}{A.4 PathHydrate}
\OutlineItemTight{1.5em}{appendix:alg:rezip}{A.5 ReZip}
\OutlineItemTight{1.5em}{appendix:alg:evaluation}{A.6 Evaluation-facing Metrics}
\OutlineTop{0pt}{appendix:diagnostic_results}{B.\,Additional Experiments}
\OutlineItem{1.5em}{appendix:compression_accounting}{B.1 Compression Fidelity and Contract Robustness}
\OutlineItem{3em}{exp:compression-accounting}{\textbf{(RQ5)} Does compression reduce active storage while limiting downstream recovery?}
\OutlineItemGap{3em}{exp:contract-robustness}{\textbf{(RQ6)} How robust is contract extraction?}
\OutlineItem{1.5em}{appendix:scale_overlap}{B.2 Retrieval Scalability and Ambiguity}
\OutlineItemGap{3em}{exp:library-scale}{\textbf{(RQ7)} How does retrieval scale as the skill library grows?}
\OutlineItem{1.5em}{appendix:overlap_applicability}{B.3 Procedural Overlap and Domain Applicability}
\OutlineItem{3em}{exp:procedural-overlap}{\textbf{(RQ8)} How does procedural overlap affect compression?}
\OutlineItemGap{3em}{exp:cross-domain}{\textbf{(RQ9)} Does compression generalize across procedural domains?}
\OutlineItem{1.5em}{appendix:budget_results}{B.4 Hydration Quality and Context Compactness}
\OutlineItem{3em}{exp:hydration-budget}{\textbf{(RQ10)} How does PathHydrate balance task quality and context budget?}
\OutlineItemGap{3em}{exp:context-compactness}{\textbf{(RQ11)} Is hydrated context compact under the default budget?}
\OutlineItem{1.5em}{appendix:cost_results}{B.5 System Cost across the Skill Lifecycle}
\OutlineItem{3em}{exp:offline-construction}{\textbf{(RQ12)} How does offline structural construction scale?}
\OutlineItem{3em}{exp:cost-decomposition}{\textbf{(RQ13)} How costly are task-time retrieval and rendering?}
\OutlineItemGap{3em}{exp:end-to-end-cost}{\textbf{(RQ14)} How does \method affect end-to-end agent cost?}
\OutlineItem{1.5em}{appendix:rezip_results}{B.6 Streaming ReZip Maintenance}
\OutlineItemGap{3em}{exp:streaming-maintenance}{\textbf{(RQ15)} Can ReZip maintain an evolving skill library?}
\OutlineItem{1.5em}{appendix:statistical_results}{B.7 Reliability across Runs and Backbones}
\OutlineItem{3em}{exp:statistical-reliability}{\textbf{(RQ16)} Are gains stable across repeated agent runs?}
\OutlineItemGap{3em}{exp:cross-backbone}{\textbf{(RQ17)} Does \method generalize across LLM backbones?}
\OutlineItem{1.5em}{appendix:failure_analysis}{B.8 Failure Analysis}
\OutlineItemGap{3em}{exp:failure-attribution}{\textbf{(RQ18)} Where do the remaining failures originate?}
\OutlineTopTight{0pt}{appendix:case_studies}{C.\,Case Studies}
\OutlineItemTight{1.5em}{appendix:case_ambiguity}{Case 1: Resolving Skill-Level Ambiguity}
\OutlineItemTight{1.5em}{appendix:case_contract}{Case 2: Preserving Contracts during Compression}
\OutlineItemTight{1.5em}{appendix:case_rezip}{Case 3: Maintaining Compression under Library Evolution}
\OutlineTopCompact{0pt}{appendix:exp_details}{D.\,Experimental Details}
\OutlineTopCompact{0pt}{appendix:detailed_related_work}{E.\,Detailed Related Work}
\OutlineTopTight{0pt}{appendix:prompt}{F.\,Prompts}
\OutlineItemTight{1.5em}{prompt:section-contract}{Section Role and Contract Extraction}
\OutlineItemTight{1.5em}{prompt:signature-canonicalization}{Cross-Section Signature Canonicalization}
\OutlineItemTight{1.5em}{prompt:skillsbench-anchoring}{\textsc{SkillsBench} Task Anchoring}
\OutlineItemTight{1.5em}{prompt:alfworld-anchoring}{\alfworld Task Anchoring}
\OutlineItemTight{1.5em}{prompt:skillsbench-execution}{\textsc{SkillsBench} Execution-Context Interface}
\OutlineItemTight{1.5em}{prompt:alfworld-execution}{\alfworld Action and Source-Expansion Interface}

\clearpage
\twocolumn
\section{Algorithm}
\label{appendix:alg}

\subsection{\method Workflow}
\label{appendix:alg:overview}
We summarize the full \method pipeline in Algorithm~\ref{alg:skillzip_overview}.
The workflow first builds a raw section graph, then compresses repeated
procedural motifs, hydrates a task-specific execution context, and updates the
compressed library when new skills or execution traces arrive. Colored
badges indicate operation types: \SecTool{Blue} denotes section parsing and
section-graph construction, \ZipTool{Orange} denotes motif mining and
macro-compression, \HydTool{Purple} denotes task-time hydration and context
rendering, \UpdTool{Red} denotes incremental update and risk maintenance,
\ChkTool{Green} denotes contract checking and verifier-related validation, and
\EvalTool{Gray} denotes evaluation or logging operations.

\begin{algorithm}[H]
\small
\SetVline
\caption{\method Workflow}\label{alg:skillzip_overview}
\Input{Skill library $\mathcal{S}$, update stream $\mathcal{U}$, query $q$, profile $p$, budget $B$}
\Output{Hydrated context $C_q$, hydration log $\mathcal{L}_q$, compressed graph $\mathcal{G}_{zip}$}

\vspace{1mm}
\CmtState{\\$\mathcal{G}\gets\emptyset,\ \mathcal{M}\gets\emptyset,\ \mathcal{B}_{res}\gets\emptyset,\ \Sigma\gets\emptyset$}{\textbf{Stage 1: Sec2Graph: Building Procedural Subgraphs}}
\ForEach{skill package $s\in\mathcal{S}$}{
  \State{$h_s\gets \SecTool{Sec2Graph}(s)$}
  \State{$\mathcal{G}\gets \SecTool{MergeSkillGraph}(\mathcal{G},h_s)$}
}
\vspace{1mm}
\CmtState{\\$\mathcal{G}\gets \SecTool{LinkCanonicalPrototypes}(\mathcal{G})$}{\textbf{Phase: Cross-skill reuse}}
\vspace{1mm}
\CmtState{\\$(\mathcal{G}_{zip},\mathcal{M})\gets \ZipTool{MotifZip}(\mathcal{G})$}{\textbf{Stage 2: MotifZip: Contract-Preserving Compression}}
\vspace{1mm}
\CmtState{\\$(C_q,\mathcal{L}_q)\gets \HydTool{PathHydrate}(q,p,B,\mathcal{G}_{zip},\mathcal{M})$}{\textbf{Stage 3: PathHydrate: Budgeted Executable Context}}
\vspace{1mm}
\CmtState{\\$\mathcal{U}_{run}\gets \mathcal{U}$}{\textbf{Stage 4: ReZip: Incremental Library Maintenance}}
\State{$\mathcal{Z}\gets(\mathcal{G}_{zip},\mathcal{M},\mathcal{B}_{res},\Sigma)$}
\ForEach{new skill or execution trace $u\in\mathcal{U}_{run}$}{
  \State{$\mathcal{Z}\gets \UpdTool{ReZip}(\mathcal{Z},u)$}
}
\State{$(\mathcal{G}_{zip},\mathcal{M},\mathcal{B}_{res},\Sigma)\gets\mathcal{Z}$}
\State{\Return $C_q,\mathcal{L}_q,\mathcal{G}_{zip}$}
\end{algorithm}

\newpage
\subsection{Sec2Graph}
\label{appendix:alg:sec2graph}
Algorithm~\ref{alg:sec2graph} details how \SecTool{Sec2Graph} converts one
skill package into a typed procedural graph. The key idea is to keep all
execution roles explicit: each section carries a role, input/output signature,
resource reference, guard or verifier condition, and source pointer.

\begin{algorithm}[H]
\small
\SetVline
\caption{Sec2Graph}\label{alg:sec2graph}
\Input{Skill package $s$ with markdown, scripts, schemas, tests, and resources}
\Output{Skill procedural subgraph $h_s=(V_s,E_s,r_s,t_s)$}

\vspace{1mm}
\CmtState{\\$\mathcal{B}\gets \SecTool{SegmentSkill}(s)$; $V_s\gets\emptyset,\ E_s\gets\emptyset$}{\textbf{Phase: Section grounding}}
\ForEach{candidate span $b\in\mathcal{B}$}{
  \State{$\tau_b\gets \SecTool{InferRole}(b)$}
  \vspace{1mm}
  \CmtState{\\$(X_b,Y_b)\gets \SecTool{ExtractSignature}(b)$; $R_b\gets \SecTool{ExtractResources}(b,s)$; $G_b\gets \SecTool{ExtractGuardsVerifiers}(b)$}{\textbf{Phase: Contract extraction}}
  \State{$v_b\gets \langle \tau_b,b,X_b,Y_b,R_b,G_b,\mathrm{src}_b\rangle$}
  \State{$V_s\gets V_s\cup\{v_b\}$}
}
\vspace{1mm}
\CmtState{\\$E_s\gets E_s\cup \SecTool{WeakOrderEdges}(V_s)$}{\textbf{Phase: Procedural subgraph construction}}
\ForEach{operation node $v\in V_s$}{
  \State{$E_s\gets E_s\cup \SecTool{BindInputs}(v,V_s)$}
  \State{$E_s\gets E_s\cup \SecTool{AttachRequirements}(v,V_s)$}
  \State{$E_s\gets E_s\cup \SecTool{AttachVerifiers}(v,V_s)$}
  \State{$E_s\gets E_s\cup \SecTool{AttachRepairs}(v,V_s)$}
}
\State{$E_s\gets E_s\cup \SecTool{SkillMembership}(s,V_s)$}
\State{$(r_s,t_s)\gets \SecTool{SelectEndpoints}(V_s,E_s)$}
\State{\Return $(V_s,E_s,r_s,t_s)$}
\end{algorithm}

\SecTool{SegmentSkill} uses headings, lists, code blocks, tool references, test
files, and warning phrases to create high-recall section candidates.
\SecTool{AttachVerifiers} is deliberately conservative: an operation is linked
only to verifiers that can be reached through the skill order, explicit tests,
or matching output signatures.

\newpage
\subsection{MotifZip}
\label{appendix:alg:motifzip}
Algorithm~\ref{alg:motifzip} gives the contract-preserving compression
procedure used by \ZipTool{MotifZip}. The algorithm does not cluster text spans
directly. It first searches within compatible typed signatures, then accepts a
motif only if its boundary, contract, and verifier paths remain valid after
replacement.

\begin{algorithm}[H]
\small
\SetVline
\caption{MotifZip}\label{alg:motifzip}
\Input{Raw section graph $\mathcal{G}$, minimum support $m$, candidate budget $K$}
\Output{Compressed graph $\mathcal{G}_{zip}$, macro dictionary $\mathcal{M}$}

\vspace{1mm}
\CmtState{\\$\mathcal{C}\gets\emptyset,\ \mathcal{A}_{zip}\gets\emptyset,\ \mathcal{M}\gets\emptyset$}{\textbf{Phase: Interface-aware motif mining}}
\State{$\mathcal{B}\gets \ZipTool{BucketBySignature}(\mathcal{G})$}
\ForEach{bucket $B_i\in\mathcal{B}$}{
  \State{$\mathcal{C}\gets \mathcal{C}\cup \ZipTool{GrowMotifs}(B_i,\mathcal{G})$}
}
\State{$\mathcal{C}\gets$ top-$K$ candidates ranked by support and preliminary gain}
\vspace{1mm}
\CmtState{\\$\mathcal{A}_{zip}\gets\emptyset$}{\textbf{Phase: Contract-preserving grammar construction}}
\ForEach{candidate motif $g\in\mathcal{C}$}{
  \State{$\Omega_g\gets \ZipTool{FindOccurrences}(g,\mathcal{G})$}
  \State{\textbf{if} $|\Omega_g|<m$ \textbf{then continue}}
  \State{$(I_g,O_g,\chi_g)\gets \ZipTool{BuildContract}(g,\Omega_g)$}
  \State{\textbf{if} \ChkTool{BoundaryClear}$(I_g,O_g)=\mathrm{false}$ \textbf{then continue}}
  \State{\textbf{if} \ChkTool{SignatureStable}$(\chi_g,\Omega_g)=\mathrm{false}$ \textbf{then continue}}
  \State{\textbf{if} \ChkTool{DependencyClosed}$(g,\chi_g,\mathcal{G})=\mathrm{false}$ \textbf{then continue}}
  \State{\textbf{if} \ChkTool{VerifierReachable}$(g,\chi_g,\mathcal{G})=\mathrm{false}$ \textbf{then continue}}
  \State{$\Delta(g)\gets \ZipTool{CompressionGain}(g,\Omega_g,\chi_g)$}
  \State{\textbf{if} $\Delta(g)\le 0$ \textbf{then continue}}
  \State{$\mathcal{A}_{zip}\gets \mathcal{A}_{zip}\cup\{(g,\Omega_g,I_g,O_g,\chi_g,\Delta(g))\}$}
}
\vspace{1mm}
\CmtState{\\$\mathcal{A}_{zip}\gets$ motifs in descending $\Delta(g)$}{\textbf{Phase: Macro selection and representation}}
\vspace{1mm}
\CmtState{\\$\mathcal{G}_{zip}\gets\mathcal{G}$}{\textbf{Phase: Graph grammar rule}}
\ForEach{$(g,\Omega_g,I_g,O_g,\chi_g,\Delta(g))\in\mathcal{A}_{zip}$}{
  \If{$\ZipTool{NoConflict}(g,\mathcal{M})$}{
    \State{$M_g\gets \ZipTool{CreateMacro}(g,\Omega_g,I_g,O_g,\chi_g)$}
    \State{$\mathcal{G}_{zip}\gets \ZipTool{ReplaceOccurrences}(\mathcal{G}_{zip},\Omega_g,M_g)$}
    \State{$\mathcal{M}\gets \mathcal{M}\cup\{M_g\}$}
  }
}
\State{\Return $\mathcal{G}_{zip},\mathcal{M}$}
\end{algorithm}

\ChkTool{BoundaryClear} ensures that the macro exposes the same external inputs
and outputs as the source motif. \ChkTool{DependencyClosed} checks that every
cut dependency is represented by a macro port or remains internal to the macro.
\ChkTool{VerifierReachable} ensures that state-changing operations still have
an internal or downstream verifier after compression. These checks are what
make the macro a reversible procedural rewrite rather than a lossy summary.
Operationally, an \textsc{Operation} section is treated as state-changing when
its extracted contract declares an output write, a mutation of an external
resource or environment, or a persistent tool-side effect. If the effect
metadata are missing or ambiguous, MotifZip conservatively treats the operation
as state-changing and requires a reachable verifier before compression.
The $\mathrm{Risk}(g)$ term in \ZipTool{CompressionGain} is evaluated only
after these validity checks. It increases for otherwise compatible occurrences
with missing or underspecified guards and resources, incomplete I/O bindings,
or weak verifier support. Contradictory signatures or verifier
bindings are not traded against compression gain: \ChkTool{SignatureStable}
rejects them before scoring, so the corresponding source occurrences remain
explicit and available for source-level hydration.

\subsection{PathHydrate}
\label{appendix:alg:pathhydrate}
Algorithm~\ref{alg:pathhydrate} shows how \HydTool{PathHydrate} turns the
compressed graph into a task-specific context. The runtime implementation first
constructs task-aware anchors, maps them to section seeds, concentrates the
seeds with fused skill evidence, and then compiles a connected executable
context with scaffolds, pruning, filling, closure repair, and macro-level
rendering decisions.

\begin{algorithm}[H]
\small
\SetVline
\caption{PathHydrate}\label{alg:pathhydrate}
\Input{Query $q$, profile $p$, budget $B$, compressed graph $\mathcal{G}_{zip}$, macro dictionary $\mathcal{M}$}
\Output{Rendered execution context $C_q$ and hydration log $\mathcal{L}_q$}

\vspace{1mm}
\CmtState{\\$z_q\gets \HydTool{AnalyzeTask}(q,p)$}{\textbf{Phase: Task anchoring}}
\State{$\mathcal{Q}_q\gets \HydTool{BuildSeedQueries}(z_q,q)$}
\vspace{1mm}
\CmtState{\\$\mathcal{A}_q\gets\emptyset$}{\textbf{Phase: Dual-level seed fusion}}
\ForEach{retrieval query $r\in\mathcal{Q}_q$}{
  \State{$S_r(v)\gets \HydTool{DenseSectionScore}(r,q,v)$ for $v\in\mathcal{G}_{zip}$}
  \State{$\mathcal{A}_q\gets \mathcal{A}_q\cup \HydTool{AdaptiveTopSeeds}(S_r)$}
}
\State{$R_{doc}\gets \HydTool{RankSkillDocs}(z_q,q,\mathcal{G}_{zip})$}
\State{$R_{node}\gets \HydTool{RankSkillsByNodeMax}(\mathcal{A}_q,\mathcal{G}_{zip})$}
\State{$\mathcal{K}_q\gets \HydTool{RRFConcentrate}(R_{doc},R_{node})$}
\State{$\mathcal{A}_q\gets \HydTool{ConcentrateSeeds}(\mathcal{A}_q,\mathcal{K}_q)$}
\vspace{1mm}
\CmtState{\\$P\gets \HydTool{ConnectSeeds}(\mathcal{A}_q,\mathcal{G}_{zip},q)$}{\textbf{Phase: Constrained subgraph search}}
\vspace{1mm}
\CmtState{\\$P\gets \HydTool{AttachScaffold}(P,\mathcal{G}_{zip})$}{\textbf{Phase: Scaffold repair and context filling}}
\State{$P\gets \HydTool{AttachReachableVerifiers}(P,\mathcal{G}_{zip})$}
\State{$P\gets \HydTool{RescorePrune}(P,z_q,q,B)$}
\State{$P\gets \HydTool{RoleBudgetPrune}(P)$}
\State{$P\gets \HydTool{BreadthFill}(P,\mathcal{K}_q,\mathcal{G}_{zip},q,B)$}
\State{$P\gets \HydTool{DepthFill}(P,\mathcal{G}_{zip},q,B)$}
\State{$(P,\xi_q)\gets \ChkTool{RepairClosure}(P,\mathcal{G}_{zip},B)$}
\State{$P\gets \HydTool{EnforceBudget}(P,\xi_q,B)$}
\vspace{1mm}
\CmtState{\\$U\gets \HydTool{ChooseMacroRenderLevels}(P,\mathcal{M},\xi_q)$}{\textbf{Phase: Macro-level hydration}}
\vspace{1mm}
\CmtState{\\$C_q\gets \HydTool{RenderContract}(U,p,B)$}{\textbf{Phase: Execution-context rendering}}
\State{$\mathcal{L}_q\gets \HydTool{RecordHydration}(q,z_q,P,U,C_q,\mathcal{M},\xi_q)$}
\State{\Return $C_q,\mathcal{L}_q$}
\end{algorithm}

\HydTool{AttachScaffold} adds the same-skill input, precondition, resource,
failure, verifier, and output companions around selected operations.
\HydTool{BreadthFill} uses fused skill-level evidence to cover multi-skill
tasks, while \HydTool{DepthFill} uses remaining budget to add useful sections
from the dominant skill. Both filling procedures return early once the selected
context covers the task anchors, closes required dependencies, and keeps a
verifier reachable; they do not exhaust the remaining budget by default.
\ChkTool{RepairClosure} restores pruned dependencies or verifier hooks when
possible and records any infeasible role in $\xi_q$.
\HydTool{ChooseMacroRenderLevels} keeps the original progressive-hydration
interface: a macro can be rendered as name, contract, outline, or full source
when its contract is sufficient.

\subsection{ReZip}
\label{appendix:alg:rezip}
Algorithm~\ref{alg:rezip} maintains the compressed library under two update
signals. New skills are assimilated through existing macros before residual
motifs are considered for promotion; execution traces revise macros that
repeatedly require expansion or cause downstream failures.

\begin{algorithm}[H]
\small
\SetVline
\caption{ReZip}\label{alg:rezip}
\Input{Library state $\mathcal{Z}_t=(\mathcal{G}_{zip}^{t},\mathcal{M}_t,\mathcal{B}_{res}^{t},\Sigma_t)$, update event $u_t$}
\Output{Updated library state $\mathcal{Z}_{t+1}$}

\If{$u_t$ is a new skill package}{
  \vspace{1mm}
  \CmtState{\\$h_{u_t}\gets \SecTool{Sec2Graph}(u_t)$}{\textbf{Phase: New-skill assimilation}}
  \State{$\Omega\gets \UpdTool{MatchMacros}(h_{u_t},\mathcal{M}_t)$}
  \State{$(h'_{u_t},\mathcal{R}_{u_t})\gets \UpdTool{AssimilateSkill}(h_{u_t},\Omega)$}
  \State{$\mathcal{G}_{zip}^{t}\gets \UpdTool{InsertSkill}(\mathcal{G}_{zip}^{t},h'_{u_t})$}
  \State{$\mathcal{B}_{res}^{t}\gets\mathcal{B}_{res}^{t}\cup\mathcal{R}_{u_t}$}
  \ForEach{residual motif $r\in\mathcal{B}_{res}^{t}$}{
    \If{$\operatorname{supp}_t(r)\ge m$ \textbf{and} $\Delta(r)>0$ \textbf{and} $\mathrm{Valid}_{\chi}(r)$}{
      \State{$M_r\gets \ZipTool{CreateMacro}(r,\mathcal{B}_{res}^{t})$}
      \State{$\mathcal{M}_t\gets\mathcal{M}_t\cup\{M_r\}$}
      \State{$\mathcal{G}_{zip}^{t}\gets\ZipTool{ReplaceOccurrences}(\mathcal{G}_{zip}^{t},r,M_r)$}
      \State{$\mathcal{B}_{res}^{t}\gets\mathcal{B}_{res}^{t}\setminus\operatorname{Occ}_t(r)$}
    }
  }
}
\If{$u_t$ is an execution trace}{
  \vspace{1mm}
  \CmtState{\\$\Sigma_t\gets \UpdTool{UpdateMacroStats}(\Sigma_t,\mathcal{M}_t,u_t)$}{\textbf{Phase: Execution-aware macro revision}}
  \ForEach{macro $M\in\mathcal{M}_t$}{
    \State{$\rho_t(M)\gets \UpdTool{MacroRisk}(\Sigma_t(M))$}
    \If{$\rho_t(M)\ge\eta$}{
      \State{$(\mathcal{G}_{zip}^{t},\mathcal{M}_t,\Sigma_t)\gets
      \UpdTool{ReviseMacro}(M,\mathcal{G}_{zip}^{t},\mathcal{M}_t,\Sigma_t)$}
    }
  }
}
\State{$\mathcal{Z}_{t+1}\gets(\mathcal{G}_{zip}^{t},\mathcal{M}_t,\mathcal{B}_{res}^{t},\Sigma_t)$}
\State{\Return $\mathcal{Z}_{t+1}$}
\end{algorithm}

\ChkTool{ValidContract} applies the same port, dependency-closure, and verifier
tests as \ZipTool{MotifZip}. \UpdTool{ReviseMacro} increases hydration detail
before splitting or retiring a macro, retains its source expansion rule, and
resets the evidence associated with the revised rule. The incremental update
therefore preserves the same lifting invariant as offline compression.

\newpage
\subsection{Evaluation-facing Metrics}
\label{appendix:alg:evaluation}
Algorithm~\ref{alg:evaluate_skillzip} describes how the framework-level logs
are converted into the empirical quantities used in evaluation. The algorithm
is intentionally placed after the three main modules and ReZip: it does not add
a new retrieval mechanism, but checks whether the compressed representation
preserves the structural conditions claimed by \method.

\begin{algorithm}[H]
\small
\SetVline
\caption{EvaluateSkillZip}\label{alg:evaluate_skillzip}
\Input{Evaluation queries $\mathcal{Q}$, raw graph $\mathcal{G}$, compressed graph $\mathcal{G}_{zip}$, macro dictionary $\mathcal{M}$, executor $\mathsf{Exec}$, profile $p$, budget $B$}
\Output{Per-query evaluation log $\mathcal{L}_{eval}$ and aggregate report $\mathcal{R}$}

\vspace{1mm}
\CmtState{\\$\mathcal{L}_{eval}\gets\emptyset$}{\textbf{Phase: Initialize per-query evaluation log}}
\ForEach{query $q\in\mathcal{Q}$}{
  \vspace{1mm}
  \CmtState{\\$(C_q,\mathcal{L}_q)\gets \HydTool{PathHydrate}(q,p,B,\mathcal{G}_{zip},\mathcal{M})$}{\textbf{Phase: Execute paired contexts}}
  \State{$C_q^{full}\gets \EvalTool{RenderUncompressedContext}(q,\mathcal{G})$}
  \State{$(y_q,\tau_q)\gets \EvalTool{Execute}(\mathsf{Exec},q,C_q)$}
  \State{$(y_q^{full},\tau_q^{full})\gets \EvalTool{Execute}(\mathsf{Exec},q,C_q^{full})$}
  \vspace{1mm}
  \CmtState{\\$d_q\gets \ChkTool{EvalDepPreserve}(\mathcal{L}_q,\mathcal{G})$}{\textbf{Phase: Check structural preservation}}
  \State{$v_q\gets \ChkTool{EvalVerifierReach}(\mathcal{L}_q,\mathcal{G}_{zip})$}
  \State{$m_q\gets \EvalTool{MacroExpansionRequired}(\mathcal{L}_q)$}
  \State{$f_q\gets \EvalTool{FullSourceFallback}(\mathcal{L}_q)$}
  \State{$i_q\gets \EvalTool{DownstreamInflation}(\tau_q,\tau_q^{full})$}
  \State{$t_q\gets 1-\HydTool{TokenCost}(C_q)/\HydTool{TokenCost}(C_q^{full})$}
  \State{$\mathcal{L}_{eval}\gets \mathcal{L}_{eval}\cup\{(q,y_q,t_q,d_q,v_q,m_q,f_q,i_q)\}$}
}
\vspace{1mm}
\CmtState{\\$\mathcal{R}\gets \EvalTool{AggregateMetrics}(\mathcal{L}_{eval})$}{\textbf{Phase: Aggregate task and structural metrics}}
\State{\Return $\mathcal{L}_{eval},\mathcal{R}$}
\end{algorithm}

\EvalTool{AggregateMetrics} reports task success, token reduction, dependency
preservation, verifier reachability, macro expansion, full-source fallback,
and downstream execution inflation. The last five metrics are structural
counterparts of the
constraints enforced during \ZipTool{MotifZip} and \HydTool{PathHydrate}; they
make it possible to test whether a compact context remains executable rather
than merely shorter. Downstream inflation compares paired compressed and
uncompressed executions of the same task under the same executor seed. We use
$J(\tau)=n_{\mathrm{repair}}(\tau)+n_{\mathrm{tool}}(\tau)
+2n_{\mathrm{vfail}}(\tau)$ and report
$100\max\{0,J(\tau_q)-J(\tau_q^{full})\}/\max\{1,J(\tau_q^{full})\}$.
Macro expansion and full-source fallback are aggregated as query-level
indicators, with fallback treated as a strict subset of expansion.

\clearpage

\section{Additional Experiments}
\label{appendix:diagnostic_results}

\subsection{Compression Fidelity and Contract Robustness}
\label{appendix:compression_accounting}

\phantomsection\label{exp:compression-accounting}
\myparagraphquestion{(RQ5) Does compression reduce active storage while limiting downstream recovery}
We further audit the storage and execution cost of graph compression on the
same 1K-skill library. The goal is to distinguish useful compression from
compression that only makes the stored graph smaller by pushing missing
information to downstream recovery. Therefore, we jointly measure active
storage, retained source packages, fallback expansion, and downstream inflation
(DI). Active storage counts the graph representation
and macro dictionary used during retrieval and hydration, while the original
source packages remain available only for reversible expansion. All
representations reuse the same cached section contracts, so the comparison
isolates representation and compression rather than LLM extraction.
As shown in Table~\ref{tab:compression_accounting}, \method reduces active
storage from 18.6 MB to 5.4 MB, corresponding to a 71.0\% reduction, while the
112.4 MB source library remains unchanged for reversible expansion. The gain
therefore comes from a smaller active procedural representation rather than
discarding source knowledge. Exact-text deduplication only reduces active
storage to 13.0 MB, whereas generic graph grammar reaches 6.4 MB but raises
fallback and DI to 14.8\% and 14.6\%, respectively.

Contract validation prevents this storage gain from being repaid during
execution. Removing the checks saves only another 0.5 MB, yet increases
fallback from 7.2\% to 24.7\% and DI from 2.7\% to 23.4\%. Thus, \method
achieves most of the available storage reduction while keeping downstream
recovery low. Appendix~\ref{exp:offline-construction} separately measures the
local construction cost of obtaining this representation.

\begin{table}[H]
\centering
\caption{Active-storage compression and downstream recovery on the 1K-skill
\textsc{SkillsBench} library. Source MB is retained for reversible expansion;
DI is measured against raw-graph execution.}
\label{tab:compression_accounting}
\setlength{\tabcolsep}{2.4pt}
\begin{adjustbox}{max width=\columnwidth}
\begin{tabular}{lcccc}
\toprule
\textbf{Representation} & \textbf{Active MB$\downarrow$}
& \textbf{Source MB} & \textbf{Fallback (\%)$\downarrow$}
& \textbf{DI (\%)$\downarrow$} \\
\midrule
Raw section graph     & 18.6 & 112.4 & 0.0  & 0.0 \\
Exact-text dedup.     & 13.0 & 112.4 & 2.1  & 1.8 \\
Generic graph grammar & 6.4  & 112.4 & 14.8 & 14.6 \\
\method w/o checks    & \textbf{4.9} & 112.4 & 24.7 & 23.4 \\
\textbf{\method}      & 5.4  & 112.4 & \textbf{7.2} & \textbf{2.7} \\
\bottomrule
\end{tabular}
\end{adjustbox}
\end{table}

\newpage

\phantomsection
\label{appendix:contract_results}

\phantomsection\label{exp:contract-robustness}
\myparagraphquestion{(RQ6) How robust is contract extraction}
To evaluate whether contract extraction supports reliable procedural
abstraction, we measure both field-level extraction quality and downstream
robustness under synthetic contract corruption. We first compare the extracted
interface, execution, verification, and provenance fields against human
annotations. We then remove or replace contract fields at increasing rates to
trace how extraction errors propagate to structural fidelity, source recovery,
and task reward.
Figure~\ref{fig:contract_extraction} shows a macro-average F1 of 91.6 and exact
match of 84.6. Source pointers are easiest to recover, reaching 96.2 F1 and
93.5 exact match, while preconditions and guards are most difficult at
88.7/79.6 because they are often implicit rather than stated as explicit
constraints.
Table~\ref{tab:contract_corruption} connects these extraction errors to
execution. With 10\% corruption, reward decreases only from 33.3 to 31.6 and
DPR/VR remain 96.8/95.1. As corruption increases, PathHydrate conservatively
expands more macros and restores more source sections. At 40\%, expansion and
fallback reach 41.3\% and 34.7\%, while reward falls to 24.1 and VR to 77.2.
The gradual increase in source recovery limits mild errors, but the eventual
degradation confirms that contract quality remains a central requirement for
reliable procedural abstraction.

\begin{figure}[H]
\centering
\includegraphics[width=\columnwidth]{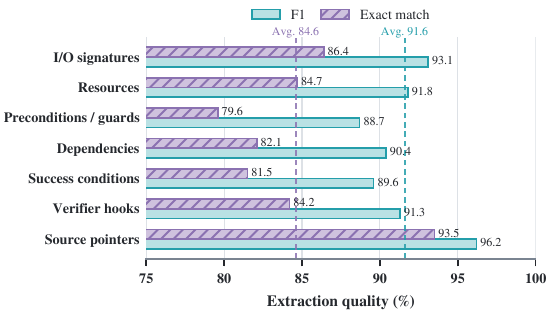}
\caption{Contract-extraction quality on the annotated subset. Bars show
field-level F1 and exact match; dashed lines denote macro averages.}
\label{fig:contract_extraction}
\end{figure}

\begin{table}[H]
\centering
\caption{Robustness under synthetic contract corruption.
Noise is the percentage of contract fields removed or replaced. Expansion and
fallback are query-level rates, with fallback requiring original-source
restoration.}
\label{tab:contract_corruption}
\setlength{\tabcolsep}{3.0pt}
\begin{adjustbox}{max width=0.9\columnwidth}
\begin{tabular}{cccccc}
\toprule
\textbf{Noise} & \textbf{R$\uparrow$} & \textbf{DPR$\uparrow$}
& \textbf{VR$\uparrow$} & \textbf{Expansion (\%)$\downarrow$}
& \textbf{Fallback (\%)$\downarrow$} \\
\midrule
0\%  & 33.3 & 99.2 & 98.7 & 14.8 & 7.2 \\
10\% & 31.6 & 96.8 & 95.1 & 19.7 & 11.8 \\
20\% & 29.2 & 92.1 & 89.4 & 27.6 & 18.9 \\
40\% & 24.1 & 81.5 & 77.2 & 41.3 & 34.7 \\
\bottomrule
\end{tabular}
\end{adjustbox}
\end{table}

\newpage

\subsection{Retrieval Scalability and Ambiguity}
\label{appendix:scale_overlap}

\phantomsection\label{exp:library-scale}
\myparagraphquestion{(RQ7) How does retrieval scale as the skill library grows}
To evaluate whether \method remains reliable as the skill library grows, we
keep the evaluation queries fixed and expand the candidate library from 200 to
100K skills. This setting separates the effect of library scale from query
difficulty: the tasks to be solved remain the same, while the retriever must
rank the correct procedural context among an increasingly large set of similar
or partially overlapping skills. We report Ret@1, similar-skill confusion,
compression ratio, and online retrieval--hydration latency.

As shown in Table~\ref{tab:scale_results}, \method degrades much more slowly
than \textsc{SkillDAG}. When the library grows from 200 to 100K skills,
\textsc{SkillDAG} Ret@1 drops from 72.1 to 41.8, a decrease of 30.3 points.
In contrast, \method Ret@1 drops from 78.3 to 65.1, a decrease of 13.2 points.
The performance gap therefore widens from 6.2 points at 200 skills to 23.3 points
at 100K skills. This indicates that the advantage of section-level
retrieval becomes larger when the library contains more distractor skills.

The confusion results explain this trend. As the library grows,
\textsc{SkillDAG} increasingly retrieves a topically related but incorrect
skill: its similar-skill confusion rate rises from 8.5\% to 48.2\%. In
comparison, \method keeps confusion much lower, increasing only from 4.2\% to
12.4\%. This suggests that whole-skill graph retrieval is more sensitive to
overlapping skill descriptions and shared high-level routines, whereas
section-level matching can distinguish the concrete operation, dependency, and
verifier needed by the query.

The scalability results also show that compression helps keep the active
retrieval structure compact. As the library grows, the compression ratio of
\method increases from 2.31$\times$ to 4.29$\times$, reflecting more repeated
procedural structure available for reuse. At the same time, online retrieval
and hydration latency grows smoothly from 18.4 ms to 248.3 ms per query and
remains below 250 ms even at 100K skills. Overall, these results show that
\method improves large-library retrieval in two ways: it reduces ambiguity by
ranking execution-relevant sections rather than whole skill packages, and it
keeps retrieval efficient by compressing repeated procedural structure into a
compact active graph.

\begin{table}[H]
\centering
\caption{Sensitivity to skill library size on \textsc{SkillsBench}. We keep
the evaluation queries fixed while expanding the candidate skill library from
200 to 100K skills. Conf. denotes the similar-skill confusion rate (\%) for
each method. Lat. denotes \method's online retrieval and hydration latency per
query.}
\label{tab:scale_results}
\setlength{\tabcolsep}{2.5pt}
\begin{adjustbox}{max width=\columnwidth}
\begin{tabular}{rcccccc}
\toprule
\textbf{Skills} & \multicolumn{2}{c}{\textbf{Ret@1$\uparrow$}}
& \multicolumn{2}{c}{\textbf{Conf.$\downarrow$}} & \textbf{CR$\uparrow$}
& \textbf{Lat. (ms)$\downarrow$} \\
\cmidrule(lr){2-3}\cmidrule(lr){4-5}
& \textsc{SkillDAG} & \method & \textsc{SkillDAG} & \method & \method & \method \\
\midrule
200    & 72.1 & \textbf{78.3} & 8.5  & \textbf{4.2}  & 2.31$\times$ & 18.4 \\
500    & 69.4 & \textbf{76.8} & 12.3 & \textbf{4.9}  & 2.87$\times$ & 21.7 \\
1K     & 66.7 & \textbf{73.6} & 15.8 & \textbf{5.8}  & 3.46$\times$ & 27.9 \\
2K     & 61.9 & \textbf{72.4} & 21.4 & \textbf{7.2}  & 3.82$\times$ & 35.8 \\
10K    & 52.4 & \textbf{68.9} & 34.5 & \textbf{9.8}  & 4.11$\times$ & 71.6 \\
100K   & 41.8 & \textbf{65.1} & 48.2 & \textbf{12.4} & 4.29$\times$ & 248.3 \\
\bottomrule
\end{tabular}
\end{adjustbox}
\end{table}

\newpage

\subsection{Procedural Overlap and Domain Applicability}
\label{appendix:overlap_applicability}

\phantomsection\label{exp:procedural-overlap}
\myparagraphquestion{(RQ8) How does procedural overlap affect compression}
To isolate the effect of procedural overlap, we keep the library size and query
distribution fixed, and vary only how often contract-compatible routines recur
across skills. This setting tests whether \method compresses more when reusable
structure is available, while avoiding unsafe compression when overlap is low.
Starting from the same 1K-skill pool, we construct three controlled variants by
injecting increasing numbers of recurring motifs into distractor-only skills.
Each injected occurrence preserves the source motif's role topology, typed I/O,
resource family, guards, and verifier connection, but remains an
occurrence-specific node sequence with independently paraphrased surface text.
For every injected occurrence, we replace a non-target routine with the same
role profile, keeping the number of skills, section count, role distribution,
evaluation queries, and target skills unchanged. Thus, the low, medium, and
high conditions differ in procedural recurrence rather than exact-text
duplication or task difficulty; Table~\ref{tab:overlap_results} reports the
resulting mean cross-skill support.

As shown in Table~\ref{tab:overlap_results}, higher overlap leads to stronger
compression. Mean macro support increases from 2.3 to 9.7 occurrences, and CR
increases from 1.18$\times$ to 4.63$\times$. At the same time, execution
fidelity remains stable: DPR stays above 99.1, VR stays above 98.8, and reward
changes by only 1.1 points across the three overlap levels. These results show
that \method adapts its compression level to the amount of reusable procedural
structure in the library. When few contract-compatible repetitions exist,
\method leaves more sections explicit rather than forcing aggressive macros;
when overlap is high, it can compress more while preserving dependency and
verifier structure.

\begin{table}[H]
\centering
\caption{Sensitivity to procedural overlap. Macro support is the mean number
of occurrence-specific source subgraphs represented by each active macro;
canonical prototypes are not counted.}
\label{tab:overlap_results}
\vspace{-2mm} 
\setlength{\tabcolsep}{3.0pt}
\begin{adjustbox}{max width=\columnwidth}
\begin{tabular}{lccccc}
\toprule
\textbf{Overlap} & \textbf{CR$\uparrow$} & \textbf{Support$\uparrow$}
& \textbf{DPR$\uparrow$} & \textbf{VR$\uparrow$} & \textbf{R$\uparrow$} \\
\midrule
Low    & 1.18$\times$ & 2.3 & 99.4 & 99.1 & 32.4 \\
Medium & 2.37$\times$ & 4.8 & 99.3 & 98.9 & 33.1 \\
High   & 4.63$\times$ & 9.7 & 99.1 & 98.8 & 33.5 \\
\bottomrule
\end{tabular}
\end{adjustbox}
\vspace{2mm}
\end{table}

\newpage
\phantomsection\label{exp:cross-domain}
\myparagraphquestion{(RQ9) Does compression generalize across procedural domains}
To test whether section-level compression is tied to a single procedural
domain, we partition the 1,000 \textsc{SkillsBench} skills into four named
domains and a residual miscellaneous category. We also evaluate \alfworld as a
separately structured embodied skill library. As shown in
Table~\ref{tab:domain_results}, \method improves the reported task metric over
\textsc{SkillDAG} in every domain where a matched baseline is available.
Reward gains range from
4.8 to 6.7 points across the four \textsc{SkillsBench} domains, while the
\alfworld success rate improves by 12.2 points. The miscellaneous category also
reaches 30.5 reward,
suggesting that \method does not depend on one dominant skill type.
The compression ratios differ across domains, which is expected because
different libraries contain different amounts of reusable structure. Data
wrangling obtains the highest CR at 4.21$\times$, likely because loading,
schema handling, normalization, and validation routines recur across many
skills. The smaller \alfworld library has a lower CR of 1.62$\times$, but still
achieves a large reward gain, showing that compression ratio and task
improvement are related but not identical: even modest compression can help
when it exposes the right executable context. Across all domains, VR remains
high, ranging from 98.5 to 99.4. Together with the controlled-overlap results,
these findings show that section abstraction transfers across procedural
settings, while the attainable compression ratio is mainly governed by how
much contract-compatible reuse exists in each library.

\begin{table}[H]
\centering
\vspace{2mm}
\caption{Results by procedural domain. The first five rows partition all
1,000 \textsc{SkillsBench} skills; \alfworld is an independently structured
embodied skill library. A dash indicates that a category-level baseline was
not available.}
\label{tab:domain_results}
\setlength{\tabcolsep}{2.3pt}
\begin{adjustbox}{max width=0.9\columnwidth}
\begin{tabular}{lccccc}
\toprule
\textbf{Domain} & \textbf{Skills} & \textbf{CR$\uparrow$}
& \multicolumn{2}{c}{\textbf{R$\uparrow$}} & \textbf{VR$\uparrow$} \\
\cmidrule(lr){4-5}
& & & \textsc{SkillDAG} & \method & \method \\
\midrule
Data wrangling        & 164 & 4.21$\times$ & 29.7 & \textbf{36.4} & 99.1 \\
Scientific computing  & 238 & 3.74$\times$ & 25.8 & \textbf{32.1} & 98.5 \\
Software/web          & 311 & 2.96$\times$ & 30.2 & \textbf{35.0} & 98.9 \\
Finance/economics     & 147 & 3.48$\times$ & 28.4 & \textbf{34.7} & 98.6 \\
Miscellaneous         & 140 & 2.50$\times$ & --   & 30.5          & 98.5 \\
\alfworld             & 37  & 1.62$\times$ & 67.1 & \textbf{79.3} & 99.4 \\
\bottomrule
\vspace{1mm}
\end{tabular}
\end{adjustbox}
\end{table}


\newpage
\subsection{Hydration Quality and Context Compactness}
\label{appendix:budget_results}

\phantomsection\label{exp:hydration-budget}
\myparagraphquestion{(RQ10) How does PathHydrate balance task quality and context budget}
To evaluate the sensitivity of PathHydrate to its context budget, we vary the
procedural-content selection budget from 1,000 to 5,000 tokens while keeping
the skill library, retriever, and executor fixed. This experiment isolates
whether additional context continues to improve execution, or whether
PathHydrate can identify a compact sufficient subgraph before exhausting the
available budget.

As shown in Figure~\ref{fig:budget_reward_curve}, reward rises sharply from
22.4 at 1,000 tokens to 31.5 at 2,000 tokens, as the hydrated context recovers
more execution-critical dependencies and verifier conditions. Increasing the
budget to the default 3,000 tokens further raises reward to 33.3, whereas
expanding it to 5,000 tokens improves reward by only another 0.8 points. The
flattening curve shows that performance does not depend on greedily filling
the executor context: most of the useful procedural structure is already
available at the default budget. RQ11 next examines how much of this budget is
actually rendered for individual tasks.

\begin{figure}[H]
\centering
\includegraphics[width=0.92\columnwidth]{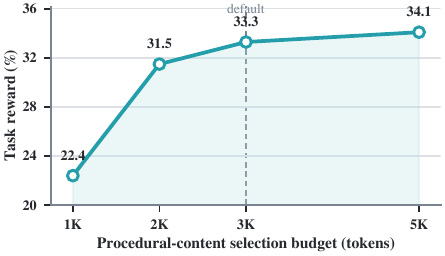}
\caption{Task reward of \method as the procedural-content selection budget
varies on \textsc{SkillsBench} with MiniMax-M2.7. The dashed line marks the
default 3,000-token budget.}
\label{fig:budget_reward_curve}
\end{figure}

\newpage
\phantomsection\label{exp:context-compactness}
\myparagraphquestion{(RQ11) Is hydrated context compact under the default budget}
To measure context compactness under the same default setting used in
Tables~\ref{tab:compression_fidelity} and~\ref{tab:ablation_skillzip}, we give
PathHydrate a 3,000-token procedural-content selection budget and measure how
much context it actually renders. We compare this delivered context with full
rendering of the same queries' top-$K$ retrieved skill packages.
As shown in Figure~\ref{fig:context_cost_soft_budget}, PathHydrate delivers
only 1,941 tokens per task on average, with a median of 1,947, even though the
budget allows up to 3,000 tokens. In contrast, rendering the top-$5$ retrieved
whole skills requires 6,958 tokens per task, and rendering the top-$12$ whole
skills requires 17,384 tokens. Thus, \method reduces delivered context by
72.1\% compared with top-$5$ whole-skill loading and by 88.8\% compared with
top-$12$ loading. The full model in Tables~\ref{tab:compression_fidelity}
and~\ref{tab:ablation_skillzip} obtains 33.3 reward with this same average
rendered context.

The task-level distribution in Figure~\ref{fig:context_token_distribution}
further shows that this saving is not an artifact of averaging. The modal
interval is 1,000--1,500 tokens, containing 26.4\% of tasks; 51.7\% of tasks
use fewer than 2,000 tokens, and 74.7\% use fewer than 2,500. These results rule
out a greedy fill-to-budget behavior. PathHydrate treats the budget as a
selection allowance rather than a quota: it terminates hydration once task
anchors are covered, dependencies are closed, and a verifier remains reachable,
and expands further only when one of these execution obligations is unresolved.
Together with the corresponding 33.3 reward, this result shows that compact
context comes from task-conditioned executable sufficiency rather than from
indiscriminate accumulation of procedural content.

\begin{figure}[t]
\centering
\includegraphics[width=\columnwidth]{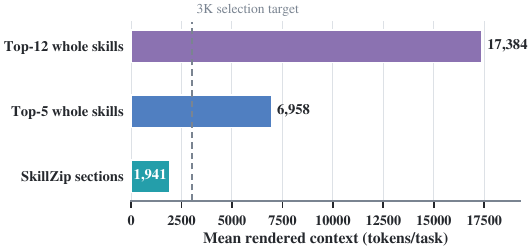}
\caption{Mean rendered context on the 1K-skill \textsc{SkillsBench} library.
Whole-skill comparisons render each retrieved package in full, whereas
\method hydrates dependency-closed sections.}
\label{fig:context_cost_soft_budget}
\vspace{1mm}
\end{figure}

\begin{figure}[t]
\centering
\includegraphics[width=0.9\columnwidth]{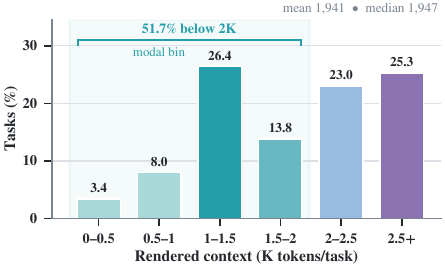}
\caption{Task-level distribution of the context rendered by \method.
Overall, 51.7\% of tasks use fewer than 2,000 tokens, and 1,000--1,500 tokens
is the modal interval.}
\label{fig:context_token_distribution}
\vspace{1mm}
\end{figure}



\subsection{System Cost across the Skill Lifecycle}
\label{appendix:cost_results}

\phantomsection\label{exp:offline-construction}

\myparagraphquestion{(RQ12) How does offline structural construction scale}
To understand the cost of maintaining \method as a persistent skill-library
representation, we separate the offline construction process into two parts.
The first part is LLM-assisted contract extraction, which reads skill packages
and produces typed section records. Its wall-clock cost depends on the chosen
model, provider, batching strategy, and request parallelism, and the extracted
records are cached after construction. Because this stage is provider-dependent,
we account for it separately and do not extrapolate its wall-clock time across
library scales. The second part is deterministic local processing, including
graph construction and MotifZip compression after the section records are
available. Table~\ref{tab:structural_build_scaling} isolates this local stage so
that the graph-algorithm cost can be measured independently of LLM inference.
As shown in Table~\ref{tab:structural_build_scaling}, the local structural
stage is lightweight for small and medium libraries: it takes 274 ms for 100
skills, 1.44 s for 1K skills, and 16.1 s for 10K skills. Even when the library
is expanded to 100K skills, corresponding to 4.77M cached section nodes, graph
construction and MotifZip finish in 178 s. This experiment measures structural
construction from cached records rather than fresh LLM extraction at 100K.
These wall-clock measurements use the single-server configuration reported in
Appendix~\ref{appendix:exp_details}.
All scales use the same cached-record format and local pipeline; the reported
wall-clock time includes fixed initialization and typed-bucket construction.
This scaling behavior is important
because MotifZip does not perform unrestricted matching over all possible
subgraphs. Instead, typed-signature bucketing in Algorithm~\ref{alg:motifzip}
groups sections by compatible roles, resource families, and I/O shapes before
motif growth. As a result, motif matching is mostly restricted to
interface-compatible candidates, and the local construction cost follows graph
size rather than exploding with arbitrary cross-library comparisons. Together
with the online latency results in Table~\ref{tab:scale_results}, these results
show that \method can be built and accessed as a persistent procedural memory
layer as the library grows.

\begin{table}[H]
\centering
\vspace{2mm}

\caption{Measured local structural construction cost after cached contract
extraction. Values report aggregate wall-clock time for graph construction and
MotifZip. LLM inference is excluded.}
\label{tab:structural_build_scaling}
\vspace{2mm}
\setlength{\tabcolsep}{5pt}
\begin{tabular}{r r r}
\toprule
\textbf{Skills} & \textbf{Section nodes} & \textbf{Graph + MotifZip$\downarrow$} \\
\midrule
100     & 6,441   & 274 ms \\
200     & 10,309  & 436 ms \\
500     & 24,333  & 1.06 s \\
1K      & 48,838  & 1.44 s \\
2K      & 96,739  & 2.90 s \\
5K      & 240,080 & 7.78 s \\
10K     & 477,681 & 16.1 s \\
100K    & 4.77 M  & 178 s \\
\bottomrule
\end{tabular}
\end{table}

\newpage
\phantomsection\label{exp:cost-decomposition}
\myparagraphquestion{(RQ13) How costly are task-time retrieval and rendering}
We next examine the local cost paid at task time. This microbenchmark starts
after task anchors are available and excludes provider-side task-analysis
inference. It isolates one retrieval-and-rendering step under the same default
setting used in our main results. The reported online latency includes
retrieval, graph closure, macro hydration, and rendering, while the rendered
tokens measure the procedural context finally delivered to the executor.
As shown in Table~\ref{tab:cost_decomposition}, \method reduces online access
latency from 41.2 ms to 27.9 ms compared with \textsc{SkillDAG}, while also
reducing rendered context from 3,103 to 1,941 tokens. This means that the
section-level compressed graph is not only more compact, but also cheaper to
query and render than the skill-level graph baseline. The rendered context of
\textsc{SkillDAG} matches that of the skill-level unit substitution in
Table~\ref{tab:ablation_skillzip}, confirming that both operate at the same
granularity. Vector Skills is faster locally, taking 12.1 ms, because it
performs dense retrieval only and does not run dependency closure or macro
hydration. However, it renders 2,834 tokens,
46.0\% more than \method, because the retrieved unit is still a whole-skill
semantic match rather than a dependency-closed procedural section. GoS and
\textsc{SkillDAG} also render more context than \method because their selected
units remain closer to coarse skill structures.

These results show the local trade-off clearly. \method is not simply the
fastest retriever in isolation, but its extra graph operations are modest and
are offset by a much smaller executable context. This matters for agent
systems because the rendered context is not used once: it is often carried
through multi-turn reasoning, tool calls, verification, and repair. Therefore,
a small increase over pure vector-retrieval latency can reduce the larger
downstream cost caused by exposing broad or noisy procedural content. Online
latency is dominated by retrieval and graph traversal rather than rendering
volume; the smaller compressed graph therefore keeps \method faster than
\textsc{SkillDAG} even after dependency closure and macro hydration.

\begin{table}[t]
\centering
\vspace{2mm}

\caption{Local retrieval and rendering cost on \textsc{SkillsBench} after task
anchors are available. Online excludes provider inference; Rendered Tok
includes fixed rendering metadata and any indivisible selected unit.}
\label{tab:cost_decomposition}
\vspace{-2mm}
\setlength{\tabcolsep}{5pt}
\begin{tabular}{lrr}
\toprule
\textbf{Method} & \textbf{Online (ms)$\downarrow$}
& \textbf{Rendered Tok$\downarrow$} \\
\midrule
Vector Skills       & \textbf{12.1} & 2,834 \\
GoS                 & 34.6 & 2,517 \\
\textsc{SkillDAG}   & 41.2 & 3,103 \\
\textbf{\method}    & 27.9 & \textbf{1,941} \\
\bottomrule
\end{tabular}
\vspace{6mm}
\end{table}

\phantomsection\label{exp:end-to-end-cost}
\myparagraphquestion{(RQ14) How does \method affect end-to-end agent cost}
Finally, we measure the full multi-turn agent trajectory, rather than only the
procedural context injected at retrieval time. For each task, we sum the
provider-reported prompt and completion tokens across all agent turns and then
average across tasks. This gives a system-level view of cost: a cleaner
procedural context may reduce not only the initial retrieval payload, but also
the number of turns, repeated prompt context, generated tokens, and tool calls
needed to complete the task. We also record wall-clock time from task launch to
harness termination, including model inference, tool and container execution,
and benchmark-harness interaction. This full-task measure is distinct from the
single-step retrieval and rendering latency reported in RQ13.
As shown in Table~\ref{tab:end_to_end_trajectory_cost}, \method improves reward
from 27.3 to 33.3 over \textsc{SkillDAG}, while substantially reducing the full
trajectory cost. Total prompt processing decreases from 2.78M to 1.47M tokens,
a 47.0\% reduction. Completion tokens decrease from 31,963 to 20,601, a
35.5\% reduction, and tool calls decrease from 36.9 to 28.9, a 21.7\%
reduction. Average task time likewise falls from 429.7 to 339.0 seconds, a
21.1\% reduction. Under the same setting, Vector Skills and GoS require 361.3
and 372.2 seconds per task, respectively, so \method remains the fastest of the
five evaluated systems. Compared with Vanilla Skills, \method also uses fewer total prompt
tokens, fewer completion tokens, and fewer tool calls while achieving a higher
reward. Thus, the cost reduction does not come from weakening the task or
shortening the response at the expense of quality; it is accompanied by better
task performance. Because full-task wall time includes provider and container
variation, we use it as end-to-end systems evidence rather than as an isolated
measure of retrieval speed.

The prompt-token results should be interpreted in two layers. The first layer
is direct context compactness: as shown in Table~\ref{tab:cost_decomposition},
\method delivers fewer procedural tokens at the retrieval step. The second
layer is trajectory shortening: cleaner section-level context reduces
irrelevant branches, missing dependencies, and unnecessary repair attempts, so
the agent repeats less accumulated context across turns and makes fewer tool
calls. This explains why the reduction in total prompt processing is larger
than the reduction in uncached prompt tokens alone. Most prompt tokens in all
systems are served from cache, so the uncached portion decreases more
moderately, from 76,880 under \textsc{SkillDAG} to 62,526 under \method, an
18.7\% reduction. The larger drop in cached and total prompt tokens reflects
that fewer turns and fewer repeated contexts are needed once the agent receives
a more compact executable procedure.

Overall, the three cost measurements support the same conclusion across the
skill lifecycle. Offline construction is a one-time and cacheable cost; local
task-time retrieval remains lightweight; and end-to-end execution becomes both
cheaper and more successful. This suggests that the main system benefit of
\method is not only reducing the number of tokens shown to the executor, but
also reducing the downstream interaction needed to use the skill library
correctly.

\begin{table}[t]
\centering
\caption{End-to-end trajectory cost on the default 1K-skill
\textsc{SkillsBench} setting with MiniMax-M2.7. Values are task averages.
Prompt counters aggregate all agent turns; task time spans model inference,
tool and container execution, and harness interaction.}
\label{tab:end_to_end_trajectory_cost}
\vspace{-2mm}
\setlength{\tabcolsep}{5pt}
\begin{tabular}{lrrr}
\toprule
\textbf{Metric} & \textbf{Vanilla Skills} & \textbf{\textsc{SkillDAG}}
& \textbf{\method} \\
\midrule
Total prompt$\downarrow$ & 2,429,237 & 2,782,696 & \textbf{1,473,532} \\
Uncached prompt$\downarrow$ & 78,081 & 76,880 & \textbf{62,526} \\
Cached prompt$\downarrow$ & 2,351,156 & 2,705,816 & \textbf{1,411,006} \\
Completion$\downarrow$ & 34,592 & 31,963 & \textbf{20,601} \\
Tool calls$\downarrow$ & 32.9 & 36.9 & \textbf{28.9} \\
Task time (s)$\downarrow$ & 464.7 & 429.7 & \textbf{339.0} \\
Reward$\uparrow$ & 17.2 & 27.3 & \textbf{33.3} \\
\bottomrule
\end{tabular}
\vspace{-2mm}
\end{table}

\subsection{Streaming ReZip Maintenance}
\label{appendix:rezip_results}

\phantomsection\label{exp:streaming-maintenance}

\myparagraphquestion{(RQ15) Can ReZip maintain an evolving skill library}
To evaluate whether ReZip can maintain compressed procedural memory as the
library evolves, we split the 1K-skill pool into an initial 50\% library and ten
equal, non-overlapping arrival batches containing the remaining 50\%. The split
and arrival order are fixed across all methods. After the fifth batch, we
introduce controlled contract drift by tightening the source-grounded verifier
condition of a fixed, domain-stratified subset of active macros while leaving
their interfaces and operations unchanged. This isolates whether maintenance
responds to changed execution obligations rather than to a different task.
After every batch, all methods are evaluated on the same frozen query set,
which covers both initial and arriving-skill procedures.

At each update, ReZip uses only cached section records and execution traces
observed up to that batch, so maintenance does not rely on future evidence and
the reported update cost excludes the one-time LLM extraction stage. We compare
it with three alternatives: keeping the original compressed graph fixed,
appending new sections without recompression, and periodically recompressing the
full library.
As shown in Table~\ref{tab:rezip_results}, the static graph degrades to 28.7
reward and 91.2 VR because old macros are not revised after verifier contracts
change. Append-only insertion partially improves reward to 30.1 and VR to
94.6, but its CR drops to 2.41$\times$ because newly recurring routines remain
as raw sections rather than being compressed. Full recompression gives the
strongest endpoint, reaching 3.71$\times$ CR, 33.6 reward, and 99.1 VR, but it
requires the full recomputation cost.
ReZip closely matches full recompression while using much lower update cost. It
achieves 3.64$\times$ CR, 33.2 reward, and 98.6 VR, with only 0.22$\times$ the
cost of periodic full recompression. It also responds to injected contract
drift within 1.4 batches on average. These results show that execution traces
provide useful maintenance signals: recurring contract-valid residuals can be
promoted into new macros, while verifier failures can demote or revise macros
whose contracts have become unsafe.

\begin{table}[H]
\centering
\vspace{2mm}
\caption{Streaming results after ten arriving-skill batches. Cost is
cumulative update cost normalized by periodic full recompression. Delay is the
mean number of batches required to recover from injected contract drift; it is
undefined for methods that do not actively repair drift.}
\label{tab:rezip_results}
\setlength{\tabcolsep}{2.3pt}
\begin{adjustbox}{max width=\columnwidth}
\begin{tabular}{lccccc}
\toprule
\textbf{Maintenance} & \textbf{CR$\uparrow$} & \textbf{Cost$\downarrow$}
& \textbf{R$\uparrow$} & \textbf{VR$\uparrow$}
& \textbf{Delay$\downarrow$} \\
\midrule
Static graph          & 3.07$\times$ & 0.00$\times$ & 28.7 & 91.2 & -- \\
Append-only sections  & 2.41$\times$ & 0.08$\times$ & 30.1 & 94.6 & -- \\
Full recompression    & \textbf{3.71$\times$} & 1.00$\times$ & \textbf{33.6} & \textbf{99.1} & \textbf{1.0} \\
\textbf{ReZip}        & 3.64$\times$ & \textbf{0.22$\times$} & 33.2 & 98.6 & 1.4 \\
\bottomrule
\end{tabular}
\end{adjustbox}
\end{table}

\newpage
\subsection{Reliability across Runs and Backbones}
\label{appendix:statistical_results}
\label{appendix:generalization_results}

\phantomsection\label{exp:statistical-reliability}
\myparagraphquestion{(RQ16) Are gains stable across repeated agent runs}
To evaluate whether the gains of \method are stable under stochastic agent
execution, we repeat each comparable method--backbone setting five times. Since
agent trajectories may vary across runs, we use a paired evaluation protocol:
for each run, \method and \textsc{SkillDAG} are paired by task or episode and
by random seed. We then compute confidence intervals by resampling paired
outcome-level differences rather than only resampling the five run-level means.
This makes the test focus on whether \method consistently improves the same
tasks or episodes, rather than whether one run happens to be favorable.

As shown in Table~\ref{tab:statistical_results}, the ordering in the main
results is preserved across all four settings. The paired task-performance
gains range from 2.8 to 12.2 points, and every 95\% confidence interval excludes
zero. On \textsc{SkillsBench}, \method improves reward over \textsc{SkillDAG}
by 6.0 points with MiniMax-M2.7 and by 6.2 points with gpt-5.2-codex, with
relatively tight confidence intervals. This indicates that the reward gains are not caused by a
small number of unstable tasks. On \alfworld, the gain is larger with
MiniMax-M2.7, where success rate improves by 12.2 points, although the interval is wider because
embodied action trajectories introduce more variation across episodes.

The most conservative setting is \alfworld with gpt-5.2-codex, where
\textsc{SkillDAG} already reaches 93.6 success and the remaining headroom is
small. Even there, \method still improves success to 96.4, with a paired gain
of 2.8 points and a significant paired permutation test ($p=.031$). These
results show that the benefit of executable section context persists across
repeated agent runs. The improvement is therefore not a single-trajectory
effect; it reflects a stable reduction in retrieval ambiguity and missing
execution context.

\begin{table}[H]
\centering
\caption{Repeated-run results. Values are mean$\pm$standard deviation over
five matched runs. $\Delta$ is the paired improvement of \method over
\textsc{SkillDAG} in the reported metric (reward or success rate), measured in
points. Confidence intervals use 1,000 paired bootstrap samples over
task or episode outcomes, and $p$-values use paired permutation tests over the
same outcome-level differences.}
\label{tab:statistical_results}
\setlength{\tabcolsep}{2.2pt}
\begin{adjustbox}{max width=\columnwidth}
\begin{tabular}{llrrrr}
\toprule
\textbf{Backbone} & \textbf{Benchmark} & \textbf{\textsc{SkillDAG}}
& \textbf{\method} & \textbf{$\Delta$ [95\% CI]} & \textbf{$p$} \\
\midrule
MiniMax-M2.7 & \textsc{SkillsBench} & 27.3$\pm$1.9 & 33.3$\pm$1.5 & +6.0 [3.8, 8.1] & .004 \\
MiniMax-M2.7 & \alfworld            & 67.1$\pm$3.1 & 79.3$\pm$2.2 & +12.2 [7.5, 16.8] & $<$.001 \\
gpt-5.2-codex & \textsc{SkillsBench} & 36.8$\pm$1.4 & 43.0$\pm$1.2 & +6.2 [4.4, 8.0] & .002 \\
gpt-5.2-codex & \alfworld            & 93.6$\pm$1.6 & 96.4$\pm$1.1 & +2.8 [0.4, 5.2] & .031 \\
\bottomrule
\end{tabular}
\end{adjustbox}
\vspace{4mm}
\end{table}

\newpage
\phantomsection\label{exp:cross-backbone}
\myparagraphquestion{(RQ17) Does \method generalize across LLM backbones}
To evaluate whether the benefit of \method depends on a particular executor,
we compare it with \textsc{Vector Skills} across six backbone LLMs and two
benchmarks. \textsc{Vector Skills} retrieves whole skills using dense semantic
similarity, while \method retrieves and hydrates execution-complete section
context. This comparison tests whether structural procedural retrieval remains
useful when the executor changes, and whether stronger LLMs can fully
compensate for the granularity mismatch of whole-skill retrieval.

As shown in Table~\ref{tab:vector_backbone_ablation}, \method outperforms
\textsc{Vector Skills} in all twelve benchmark--backbone settings. The gains
are especially large on \textsc{SkillsBench}, where tasks often require
precise tool use, file operations, or verifier-aware procedures. Across the six
backbones, \method improves reward by 19.9 to 26.2 points and achieves
2.0$\times$--3.2$\times$ relative gains over vector retrieval. The largest
relative gain appears with MiniMax-M2.7, where reward increases from 10.4 to
33.3. This suggests that smaller or less procedure-specialized executors
benefit strongly from receiving a compact and dependency-closed procedural
context, rather than a broad retrieved skill package.

The gains remain large for stronger backbones. On \textsc{SkillsBench},
Claude Sonnet 4.5 improves from 26.2 to 52.4, Gemini 3 Pro improves from 19.3
to 44.4, and gpt-5.2-codex improves from 21.5 to 43.0. These results show that
executor strength alone does not remove the need for the right retrieval unit:
even capable models can be hurt when the retrieved context contains unrelated
branches or lacks the exact verifier and dependency path required by the task.
By exposing the executable section context directly, \method reduces this
burden on the executor.

On \alfworld, the relative ratios are smaller because several vector baselines
already have high success rates. Nevertheless, \method improves every backbone.
The gains are largest for weaker backbones, with MiniMax-M2.7 increasing from
50.7 to 79.3 and Qwen 3.5 increasing from 72.9 to 87.4. For stronger models,
\method further pushes performance toward saturation, reaching 97.9 with Kimi
K2.5, 99.0 with Claude Sonnet 4.5, 99.1 with Gemini 3 Pro, and 96.4 with
gpt-5.2-codex. The pattern is consistent with the main claim: section-level
procedural context helps weaker executors by reducing retrieval ambiguity, and
still helps stronger executors by supplying compact, dependency-closed
procedures with the required verifier conditions.

Overall, the repeated-run and cross-backbone results support two reliability
claims. First, the gains of \method are statistically stable across repeated
agent executions. Second, the gains are not tied to one model family or one
executor strength level. \method acts as a procedural memory layer whose main
benefit comes from changing the retrieved unit from whole skills to executable
sections, making the context both more compact and more directly aligned with
the task.

\begin{table}[t]
\centering
\caption{Cross-backbone comparison with \textsc{Vector Skills}. Entries report
reward (\%) on \textsc{SkillsBench} and success rate (\%) on \alfworld.
$\uparrow$ is the \method/Vector ratio. The largest ratio per benchmark is
bolded, and the second largest is underlined.}
\label{tab:vector_backbone_ablation}
\setlength{\tabcolsep}{2.0pt}
\begin{adjustbox}{max width=\columnwidth}
\begin{tabular}{lcccccc}
\toprule
\multirow{2}{*}{\textbf{Backbone}}
& \multicolumn{3}{c}{\textbf{\textsc{SkillsBench}}}
& \multicolumn{3}{c}{\textbf{\alfworld}} \\
\cmidrule(lr){2-4}\cmidrule(lr){5-7}
& Vector & \method & $\uparrow$ & Vector & \method & $\uparrow$ \\
\midrule
MiniMax-M2.7      & 10.4 & 33.3 & \textbf{3.2$\times$} & 50.7 & 79.3 & \textbf{1.6$\times$} \\
Qwen 3.5          & 18.1 & 38.0 & 2.1$\times$          & 72.9 & 87.4 & \underline{1.2$\times$} \\
Kimi K2.5         & 22.4 & 47.0 & 2.1$\times$          & 89.0 & 97.9 & 1.1$\times$ \\
Claude Sonnet 4.5 & 26.2 & 52.4 & 2.0$\times$          & 94.3 & 99.0 & 1.0$\times$ \\
Gemini 3 Pro      & 19.3 & 44.4 & \underline{2.3$\times$} & 93.6 & 99.1 & 1.1$\times$ \\
gpt-5.2-codex     & 21.5 & 43.0 & 2.0$\times$          & 92.9 & 96.4 & 1.0$\times$ \\
\bottomrule
\end{tabular}
\end{adjustbox}
\end{table}

\begin{table*}[t]
\centering
\caption{Case Study 1 -- Retrieval trace for header normalization with row-count verification.
Skill-level retrieval exposes overlapping packages, whereas \method selects
the shared operation and its required verifier path.}
\label{tab:case_ambiguity}
\small
\setlength{\tabcolsep}{4.0pt}
\setlength{\fboxsep}{1.2pt}
\begin{tabularx}{\textwidth}{@{}
  >{\raggedright\arraybackslash}p{0.13\textwidth}
  >{\raggedright\arraybackslash}p{0.20\textwidth}
  >{\raggedright\arraybackslash}p{0.28\textwidth}
  >{\raggedright\arraybackslash}X@{}}
\toprule
\textbf{Retrieval unit} & \textbf{Matching evidence}
& \textbf{Selected context} & \textbf{Context consequence} \\
\midrule
\colorbox{purplemem!13}{\strut\textbf{\scriptsize Skill level}}
& Package-level similarity to both skill descriptions
& Full \emph{Clean CSV} and \emph{Pivot Table} packages
& Includes the requested routine, but also exposes missing-value repair, pivot
aggregation, and competing output rules. \\
\addlinespace[2pt]
\colorbox{kgblue!13}{\strut\textbf{\scriptsize Section level}}
& Operation anchor (\emph{normalize headers}) and verifier anchor
(\emph{row count unchanged})
& $M_{\mathrm{ingest}}$ with file/schema dependencies and the row-count
verifier
& Closes the required dependencies and verifier path without loading unrelated
downstream branches. \\
\bottomrule
\end{tabularx}
\end{table*}

\begin{table*}[t]
\centering
\caption{Case Study 2 -- Contract-aware comparison of three textually similar routines.
MotifZip accepts the two CSV occurrences because their interfaces, execution
requirements, and verifier boundaries agree, while keeping the workbook
occurrence separate.}
\label{tab:case_contract}
\small
\setlength{\tabcolsep}{4.0pt}
\begin{tabularx}{\textwidth}{@{}
  >{\raggedright\arraybackslash}p{0.13\textwidth}
  >{\raggedright\arraybackslash}p{0.19\textwidth}
  >{\raggedright\arraybackslash}p{0.25\textwidth}
  >{\raggedright\arraybackslash}p{0.20\textwidth}
  >{\raggedright\arraybackslash}X@{}}
\toprule
\textbf{Occurrence} &
\colorbox{kgblue!13}{\strut\textbf{\scriptsize Interface}} &
\colorbox{evidenceorange!15}{\strut\textbf{\scriptsize Execution}} &
\colorbox{skillgreen!15}{\strut\textbf{\scriptsize Verification}} &
\textbf{MotifZip decision} \\
\midrule
\emph{Clean CSV} &
\texttt{CSV} + delimiter $\rightarrow$ normalized table &
Infer delimiter, parse rows, normalize headers, and preserve row identity &
Schema report and reachable row-count hook &
\textcolor{skillgreen}{\textbf{Accept}} into
$M_{\mathrm{csv\mbox{-}ingest}}$ \\
\emph{Pivot Table} &
\texttt{CSV} + delimiter $\rightarrow$ normalized table &
Same ingest routine; pivot aggregation starts after the macro output port &
Schema report; downstream total verifier remains occurrence-specific &
\textcolor{skillgreen}{\textbf{Accept}} into
$M_{\mathrm{csv\mbox{-}ingest}}$ \\
\emph{Formula-Safe Workbook} &
\texttt{XLSX} $\rightarrow$ formula-preserving workbook &
Use a formula-aware resource and preserve formulas during normalization &
Formula-integrity and row-count verifiers &
\textcolor{warnred}{\textbf{Reject}} from the CSV macro; retain separately \\
\midrule
\multicolumn{5}{@{}l}{%
\colorbox{purplemem!13}{\strut\textbf{\scriptsize Reversible rewrite}}\enspace
$M_{\mathrm{csv\mbox{-}ingest}}$ keeps occurrence-specific source/port maps
and downstream verifiers.} \\
\bottomrule
\end{tabularx}
\end{table*}

\begin{figure}[t]
\centering
\includegraphics[width=\columnwidth]{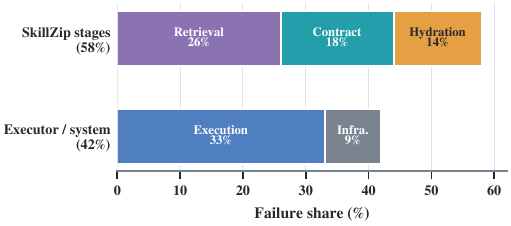}
\caption{Failure attribution grouped by the earliest blocking stage.}
\label{fig:failure_analysis}
\end{figure}

\subsection{Failure Analysis}
\label{appendix:failure_analysis}

\phantomsection\label{exp:failure-attribution}
\myparagraphquestion{(RQ18) Where do the remaining failures originate}
Under the default 1K-skill \textsc{SkillsBench} setting, we distinguish
procedural-context errors from downstream execution errors by assigning every
failed or repaired trajectory to its earliest blocking stage.
This attribution avoids counting an upstream retrieval miss again as a later
contract, hydration, or execution failure.

Figure~\ref{fig:failure_analysis} shows that stages controlled by \method account
for 58\% of the remaining failures. Retrieval is the largest source within
this group, contributing 26\%, mainly when closely related skills create
ambiguous anchors or when the needed section is not included in the initial
candidate set. Contract extraction accounts for 18\%, with errors concentrated
in implicit preconditions, resource requirements, and underspecified verifier
hooks. Hydration contributes the remaining 14\%, usually when a tight context
budget omits a required branch or when source expansion is triggered too late.

The other 42\% of failures occur after the relevant context has been retrieved
and hydrated. Agent execution is the largest single category, accounting for
33\% of cases: the executor may choose an invalid action order, misuse a tool,
or fail to recover after a verifier rejects an intermediate result.
Infrastructure failures account for 9\% and mainly include environment,
timeout, and tool-interface errors. This decomposition suggests that retrieval
ambiguity and implicit contracts are the main remaining targets for improving
\method itself, while executor and infrastructure failures require better
action control and more reliable execution environments. The smaller hydration
share also suggests that reversible source expansion usually limits structural
information loss, rather than merely shifting missing context to downstream
execution.

\newpage 
\section{Case Studies}
\label{appendix:case_studies}

To demonstrate how \method performs interpretable and execution-faithful
procedural retrieval, compression, and maintenance, we present three case
studies in Tables~\ref{tab:case_ambiguity}--\ref{tab:case_rezip}. Each case
study focuses on one key stage of the workflow: resolving retrieval ambiguity,
preserving executable contracts during compression, and updating compressed
procedures from execution evidence. Through examples involving CSV header
normalization, contract-aware CSV-ingest compression, and streaming macro
maintenance, we show how \method maps task requirements into section-level
anchors, contract-compatible motifs, dependency-closed hydrated context, and
ReZip update actions. These examples illustrate why \method can reduce active
context while keeping the procedural structure needed for verification,
execution, and source-grounded recovery. The cases are illustrative
constructions based on recurring benchmark patterns rather than additional
quantitative samples.

\vspace{3mm}
\phantomsection
\label{appendix:case_ambiguity}
\label{case:ambiguity}
\myparagraph{Case study 1: Resolving skill-level ambiguity}
Consider the task ``normalize the headers in a CSV file and verify that the
row count is unchanged.'' Both
\emph{Clean CSV} and \emph{Pivot Table} contain a file-loading and
header-normalization routine, so a skill-level retriever ranks both packages
highly. Table~\ref{tab:case_ambiguity} contrasts the resulting whole-skill
context with the section-level context compiled by \method.

Whole-skill retrieval exposes missing-value repair, pivot aggregation, and
competing output rules together with the requested routine. PathHydrate instead
maps the query to an \textsc{Operation} anchor (\emph{normalize headers}) and a
\textsc{Verifier} anchor (\emph{row count unchanged}). Section matching
recovers the shared tabular-ingest routine, while dependency closure adds the
required file input and schema requirements. The hydrated context therefore
keeps the reusable ingest macro and row-count verifier but excludes unrelated
continuations. This case illustrates that section-level retrieval reduces both
context size and ambiguity between topically similar skills.

\begin{table*}[t]
\centering
\vspace{1mm}
\caption{Case Study 3 -- Illustrative ReZip trace. Repeated compatible occurrences promote a
residual, while execution risk triggers controlled macro demotion.}
\label{tab:case_rezip}
\small
\setlength{\tabcolsep}{4.0pt}
\begin{tabularx}{\textwidth}{@{}
  >{\raggedright\arraybackslash}p{0.10\textwidth}
  >{\raggedright\arraybackslash}p{0.21\textwidth}
  >{\raggedright\arraybackslash}p{0.24\textwidth}
  >{\raggedright\arraybackslash}p{0.20\textwidth}
  >{\raggedright\arraybackslash}X@{}}
\toprule
\textbf{Stage} & \textbf{Incoming signal} & \textbf{Contract evidence}
  & \textbf{ReZip update} & \textbf{Resulting library state} \\
\midrule
\colorbox{kgblue!13}{\strut\textbf{\scriptsize Insert}}
  & A new \emph{Merge Monthly Reports} skill arrives
  & Ingest matches $M_{\mathrm{csv\mbox{-}ingest}}$; the
    align--merge--balance subgraph is unmatched
  & Reuse the ingest macro and buffer the residual
  & Known structure is compressed; novel steps remain explicit \\
\colorbox{skillgreen!15}{\strut\textbf{\scriptsize Promote}}
  & The merge residual recurs in later skills
  & Ports, dependencies, resource family, and balance verifier remain stable
  & Promote $M_{\mathrm{period\mbox{-}merge}}$
  & Later skills reuse one verified period-merge routine \\
\colorbox{warnred!13}{\strut\textbf{\scriptsize Revise}}
  & Formula-bearing tasks repeatedly expand or fail the generic export macro
  & Failures localize to the \texttt{XLSX} resource and formula-integrity
    verifier
  & Split by resource family; require full source for \texttt{XLSX}
  & CSV export remains compact; workbook safeguards are restored \\
\colorbox{purplemem!13}{\strut\textbf{\scriptsize Reuse}}
  & A future CSV or workbook query arrives
  & Task anchors identify the required resource and verification contract
  & Select the macro and hydration level by task contract
  & Compact where stable; source-expanded where evidence indicates risk \\
\bottomrule
\end{tabularx}
\end{table*}

\vspace{3mm}

\phantomsection
\label{appendix:case_contract}
\label{case:contract}
\myparagraph{Case study  2: Preserving contracts during compression}
Suppose three skills contain the near-identical instruction ``load the table,
normalize its headers, and validate the result.'' The occurrences belong to
\emph{Clean CSV}, \emph{Pivot Table}, and a \emph{Formula-Safe Workbook}
skill. Table~\ref{tab:case_contract} compares their interface, execution, and
verification contracts before compression.

Text deduplication would merge all three routines because their surface forms
are similar. MotifZip rewrites them only after checking typed ports,
dependencies, resources, and verifier paths. It therefore compresses the two
CSV occurrences into $M_{\mathrm{csv\mbox{-}ingest}}$ while retaining
occurrence-specific source maps and downstream verifier links; pivot
aggregation remains outside the macro and reconnects through its output port.
The workbook occurrence is rejected because formula preservation changes its
interface, execution resource, and verification contract. Thus, the graph
grammar accepts reuse only across a shared executable boundary and keeps every
accepted macro reversible to its source occurrences.

\newpage
\phantomsection
\label{appendix:case_rezip}
\label{case:rezip}
\myparagraph{Case study  3: Maintaining compression under library evolution}
This case follows how ReZip maintains compressed procedural memory as new
skills arrive and old abstractions become unsafe. A newly added
\emph{Merge Monthly Reports} skill reuses the existing CSV-ingest macro, but
also introduces an explicit \emph{align periods--merge--balance check}
residual. Separately, execution traces show that a generic export macro
repeatedly fails verifier checks on formula-bearing workbooks.
Table~\ref{tab:case_rezip} traces the resulting insert, promote, revise, and
reuse decisions.

ReZip first reuses the stable ingest macro and leaves the novel residual
explicit. Repeated contract-compatible occurrences then promote the residual
without full recompression. In the opposite direction, repeated expansion,
repair, or verifier failure localizes risk to the formula-bearing task family,
causing ReZip to split the export macro or raise its default hydration level.
Execution evidence therefore updates the abstraction itself: the library stays
compact where reuse is stable and source-expanded where verification indicates
risk. Across the three cases, Sec2Graph exposes sub-skill reuse, MotifZip
compresses contract-compatible occurrences, PathHydrate closes the required
execution context, and ReZip revises unsafe abstractions.

\section{Experimental Details}
\label{appendix:exp_details}

\myparagraph{Default setting and benchmarks}
We evaluate \method on two complementary agent benchmarks:
\textsc{SkillsBench} and \alfworld. \textsc{SkillsBench} tests whether
retrieved procedural knowledge helps an agent construct verifiable artifacts,
while \alfworld tests whether the same procedural-memory design supports
long-horizon interactive execution. Unless stated otherwise, the default
setting uses \textsc{SkillsBench} with the 1,000-skill library, MiniMax-M2.7,
and a 3,000-token procedural-content selection budget. The corresponding
default budget on \alfworld is 1,200 tokens. Experiments that vary the library
size, backbone, benchmark, context budget, procedural overlap, or update stream
state the changed factor explicitly and keep the remaining settings fixed.

\myparagraphunderline{\textsc{SkillsBench}}
\textsc{SkillsBench}~\cite{skillsbench} evaluates agent skills on
containerized artifact-construction tasks, including data processing, document
manipulation, software development, and related procedural domains. Each task
is paired with an executable verifier, so the final artifact is scored by
deterministic tests rather than by an LLM judge. The verifier source, excerpts,
assertions, and intermediate outcomes are withheld from task anchoring,
retrieval, hydration, and every evaluated agent; the verifier is invoked only
after execution to compute reward. Following
Graph-of-Skills~\cite{graphofskills} and
\textsc{SkillDAG}~\cite{skilldag}, we evaluate all 87 tasks under the
1,000-skill setting. Task reward (R) is the percentage of verifier tests
passed. Target source skills are used only for intrinsic retrieval evaluation.

\myparagraphunderline{\alfworld}
\alfworld~\cite{alfworld} aligns text-based interaction with embodied
household tasks derived from ALFRED~\cite{alfred} through the TextWorld
interface~\cite{textworld}. An agent must interpret a natural-language goal,
navigate rooms, manipulate objects, and complete the goal through admissible
actions. Following Graph-of-Skills~\cite{graphofskills} and
\textsc{SkillDAG}~\cite{skilldag}, we use the \texttt{valid\_seen} split and evaluate all 140 episodes. Task reward is the
percentage of episodes that reach the goal within the common step and attempt
budgets.
\vspace{2mm}

\myparagraph{Evaluation metrics}
We report end-task reward, intrinsic retrieval quality, structural fidelity,
recovery behavior, and system cost.

\myparagraphunderline{Task and retrieval quality}
For \textsc{SkillsBench}, reward (R) is the percentage of verifier tests
passed. For \alfworld, reward is the episode success rate. For intrinsic
retrieval, Ret@$k$ is the percentage of queries whose retrieved context maps to
a target source skill within the top $k$, and MRR is the mean reciprocal rank
of the first target source skill. Since \method retrieves sections and macros
rather than whole skills, we project each retrieved section to its owning
source skill and each macro occurrence to the source skill recorded in its
expansion map. We preserve the original retrieval order and remove repeated
source IDs by stable first occurrence, without additional reranking or score
aggregation. This gives whole-skill retrieval, section retrieval, and macro
retrieval the same Ret@$k$/MRR unit.

\myparagraphunderline{Compression and structural fidelity}
For compression analysis, CR is the ratio between the raw active
representation and the compressed active representation, where the compressed
representation includes the macro dictionary. Tok is the average number of
final rendered context tokens per query. The procedural-content selection
budget applies to selected procedural payloads before fixed task headers,
execution metadata, and source pointers are rendered. Whole-skill methods may
exceed the target budget when the smallest selected package is indivisible, so
we report measured final Tok rather than treating the target budget as the
observed context size.
DPR measures the fraction of required dependency relations retained in the
first compact view, and VR measures the fraction of required
operation-to-verifier paths that remain reachable. 
Expansion is the
percentage of queries for which at least one selected macro must be raised from
a compact name/contract view to an outline or full-source view. Fallback
is the percentage of queries that ultimately require an original source
section. For \method, fallback occurs after insufficient macro expansion and is
therefore a subset of expansion. Recover. is the broader query-level
rate used in Table~\ref{tab:compression_fidelity}: it records whether the
first compact view requires graph-local closure repair, macro expansion, or
source restoration before execution.

\myparagraphunderline{Downstream inflation}
We measure downstream execution inflation (DI) against a paired execution using
the corresponding raw section context. For trace $\tau$, let
$
J(\tau)=n_{\mathrm{repair}}(\tau)+n_{\mathrm{tool}}(\tau)
        +2n_{\mathrm{vfail}}(\tau),
$
where the terms count repair steps, downstream tool calls, and verifier
failures. For query $q$,
\[
\mathrm{DI}(q)=100\,
\frac{\max\{0,J(\tau_q)-J(\tau_q^{\mathrm{full}})\}}
     {\max\{1,J(\tau_q^{\mathrm{full}})\}}.
\]
We report the mean query-level percentage. DI separates compact contexts that
genuinely reduce work from contexts that only move missing information into
later repair.

\vspace{2mm}
\myparagraph{Baselines and budget alignment}
We compare \method with whole-skill disclosure, semantic retrieval,
skill-graph retrieval, and representation-level compression baselines.

\myparagraphunderline{Skill retrieval baselines}
\textbf{Vanilla Skills} exposes the available skill packages directly to a
ReAct-style agent~\cite{react}, providing a non-retrieval reference for task
quality and context cost. \textbf{Vector Skills} embeds each whole skill and
retrieves the top-ranked packages by dense semantic similarity. It tests
whether semantic retrieval alone can distinguish closely related skills.
\textbf{Graph-of-Skills (GoS)}~\cite{graphofskills} obtains semantic and
lexical seeds, diffuses relevance over a directed skill-dependency graph, and
hydrates a bounded bundle. It is the closest dependency-aware structural
retrieval baseline while still using skills as graph nodes.
\textbf{\textsc{SkillDAG}}~\cite{skilldag} represents inter-skill
dependencies, conflicts, specializations, and equivalences as typed edges, and
uses an agent-callable interface for vector matches, typed neighbors, and
conflict signals. We use it as the strongest task-time skill-graph baseline.

\myparagraphunderline{Disclosure and budget alignment}
We preserve each baseline's main disclosure policy rather than forcing all
methods to expose the same number of sections or packages. GoS uses five
initial seeds, returns at most eight skills, truncates each skill to 2,400
characters, and caps the rendered bundle at 12,000 characters.
\textsc{SkillDAG} uses top-$5$ retrieval with graph depth $2$ and allows the
agent to issue additional search or source-display calls on demand. Since
these policies do not map cleanly to a single token cap, we keep their
published retrieval settings and report measured rendered tokens and
end-to-end agent tokens.
All methods receive the same agent-visible task brief and benchmark metadata.
\method's structured task object is derived only from this shared input; no
method receives benchmark verifier code or an excerpt of its assertions.

\myparagraphunderline{Compression baselines}
For representation-level analysis, we compare \method with four compression
variants. \textbf{Exact-text section deduplication} merges normalized,
text-identical sections without checking occurrence-specific contracts.
\textbf{Text compression} applies LLMLingua-2-style task-agnostic token
compression~\cite{llmlingua2} after retrieving the same raw section context.
\textbf{Generic graph grammar} runs on the same raw section graph with the same
support threshold, candidate-window bound, macro cap, and greedy non-overlap
rewrite schedule as MotifZip. It identifies and accepts motifs using topology
and description-length gain without consulting role, resource, I/O, guard, or
verifier fields. \textbf{\method without contract checks} is a closer
controlled ablation: it retains MotifZip's signature-bucketed candidates,
scoring, and rewrite machinery, but removes the boundary, signature,
dependency, and verifier acceptance checks before admitting each positive-gain
motif.
For the text-compression baseline, we retrieve and render the same source
sections as the raw-graph representation, then apply LLMLingua-2-style
compression without MotifZip. We tune its compression threshold on development
queries to match \method's active-representation ratio and apply the same
default procedural-content selection budget at query time. The same
dependency and verifier validators are then applied to the first compressed
view. When either check fails, the corresponding original section is restored
before execution. DPR and VR are measured before restoration, while Recover.
reports the fraction of queries requiring restoration or repair. CR is measured
on the stored active representation before query-time restoration, whereas Tok
and reward are measured from the final context and execution after restoration.
Thus, CR and Tok describe different stages of the pipeline.

\vspace{2mm}

\myparagraph{Models and implementation}
We evaluate all methods under matched backbone, executor, and retrieval
settings within each comparison block.

\myparagraphunderline{Backbones and execution}
The main results use MiniMax-M2.7 (\texttt{MiniMax-M2.7}) and gpt-5.2-codex
(\texttt{gpt-5.2-codex}). The cross-backbone analysis additionally uses
Qwen 3.5 (\texttt{qwen3.5-plus}), Kimi K2.5 (\texttt{kimi-k2.5}), Claude
Sonnet 4.5 (\texttt{claude-sonnet-4-5-20250929}), and Gemini 3 Pro
(\texttt{gemini-3-pro-preview}).
On \textsc{SkillsBench}, all backbones are executed with the OpenHands agent
through the benchmark's BenchFlow Docker harness. Every method within a
backbone block receives the same task image, tools, system instructions,
two-attempt policy, and skill-library snapshot. On \alfworld, all backbones use
the same text-action runner. At each turn, the runner exposes the current
observation and admissible actions, and the model returns either one
environment action or a source-expansion request. Episodes are capped at 30
environment steps, 60 model turns, and two source expansions. The evaluated
backbone is also used for task anchoring; no auxiliary model is substituted for
decomposition or execution.

\myparagraphunderline{Retrieval, compression, and hydration settings}
All task-analysis, retrieval-control, and action-generation calls use
temperature $0$ and deterministic decoding whenever supported. Other sampling
controls remain at provider defaults. Structured task anchoring is capped at
8,000 output tokens, and each \alfworld action call at 4,096 output tokens.
\textsc{SkillsBench} otherwise follows the generation limits of its agent
harness.
We use BGE-M3 (\texttt{BAAI/bge-m3}) for dense retrieval, cap initial skill
recall at 12, and retain a lower-ranked skill only when its cosine similarity
is at least $0.45$ and at least $0.90$ of the highest score. The same encoder,
cached embeddings, and cutoff are used by all retrieval methods that require
dense similarity. For MotifZip, we scale the description-length saving and the auxiliary
$\mathrm{Reuse}$, $\mathrm{Cut}$, and $\mathrm{Risk}$ terms to comparable
ranges before scoring. We set $\alpha=0.5$, $\lambda=0.3$, and $\mu=0.2$;
because contract-invalid candidates are rejected by hard acceptance checks,
these coefficients rank only contract-valid motifs and prioritize cross-skill
reuse among them. For PathHydrate, we normalize token cost by the selection
budget and scale the remaining objective terms to comparable ranges. We use
$\eta=0.4$ and $\beta=\gamma=\delta=0.2$, emphasizing compactness among
subgraphs that already satisfy anchor coverage, dependency closure, and
verifier reachability.
\vspace{2mm}

\myparagraph{Reproducibility, scaling, and cost accounting}
We separate offline preprocessing, local graph construction, online retrieval,
and downstream execution cost.

\myparagraphunderline{Runtime and cost accounting}
\textsc{SkillsBench} is executed with BenchFlow 0.6.2 in isolated Docker
containers. Provider inference is remote and excluded from local graph-search
latency. Reported online graph latency measures retrieval, dependency closure,
and hydration after task anchors are available.
All local structural-construction timings are measured on a single server
equipped with an Intel Xeon Gold 6248R CPU and 512\,GB of memory. We report
elapsed wall-clock time under the same local implementation at every scale.
The timed graph-construction and MotifZip stage runs as a single process and is
not parallelized across skills or motif buckets. Model-side request parallelism
does not enter these measurements because they start from cached section
records and exclude LLM inference.

Offline construction has two stages. The first stage is LLM-assisted role and
contract extraction, which is cached after construction and depends on the
provider, batching strategy, and request parallelism. We therefore report it
separately from graph-algorithm scaling and do not extrapolate fresh extraction
time to the 100K setting. The second stage is local structural construction,
which starts from cached extracted records and includes procedural graph
construction and MotifZip. The corresponding 1K-skill local stage takes 1.44
seconds, while the 100K result reports the same cached-record starting point.
Each input document in this timing study is one skill package, so document and
skill counts are identical.

For end-to-end token accounting, \emph{total prompt} is the sum of
provider-reported prompt tokens over all model turns in a task. \emph{Uncached
prompt} is the subset not served from the provider cache, and \emph{cached
prompt} is the difference between total and uncached prompt tokens.
\emph{Completion} is the total number of generated tokens over the same turns.
These trajectory-level counters repeatedly include the growing interaction
history and are distinct from Tok, which measures the procedural context
rendered once for a query. We also report average tool calls from the saved
task traces and pair every method with the corresponding task reward.

\myparagraphunderline{Retrieval and library scaling}
For library-scaling experiments, we keep the evaluation queries and
target-source annotations fixed while enlarging only the candidate skill
library. Each larger library is a strict superset of the smaller one.
Additional skill packages are drawn without replacement from the same
SkillsBench-compatible source collection, deduplicated by stable package
identifier, and added in deterministic identifier order. We use the released
Graph-of-Skills pools at 200, 500, 1K, and 2K skills, and extend the same
nested-pool construction to 10K and 100K skills. The added packages are not
paired with evaluation queries and serve only as retrieval distractors. Across
all scales, the embedding backbone, retrieval budget, and evaluation protocol
remain unchanged, so Ret@1 and confusion changes reflect library growth rather
than query difficulty.

\myparagraphunderline{Repeated runs and statistical tests}
Statistical experiments use five matched runs, pairing methods by task or
episode and random seed. Confidence intervals and significance tests are
computed over paired outcome-level differences as described in
Appendix~\ref{appendix:statistical_results}.

\section{Detailed Related Work}
\label{appendix:detailed_related_work}

\myparagraph{Tool use, skill acquisition, and evaluation}
Tool-augmented agents interleave reasoning with environment actions through
prompting~\cite{react}, self-supervised tool invocation~\cite{toolformer}, or
large API corpora and benchmarks~\cite{apibank,toolllm,agentbench}. Beyond
individual calls, agents increasingly preserve successful behavior as reusable
procedural knowledge. Voyager~\cite{voyager} accumulates executable programs,
Reflexion~\cite{reflexion} and ExpeL~\cite{expel} distill feedback into reusable
experience, and Agent Workflow Memory~\cite{awm} retrieves workflows from prior
trajectories. SkillWeaver~\cite{skillweaver} discovers and hones reusable web
APIs through environment exploration, while SkillFoundry~\cite{skillfoundry}
extracts operational contracts from heterogeneous scientific resources and
iteratively validates, repairs, merges, or prunes the resulting skills. Agent
Skills~\cite{agentskills} formalizes a deployable package containing
instructions, scripts, references, and resources. A recent survey organizes
this area around skill representation, acquisition, retrieval, and
evolution~\cite{agentskillssurvey}; complementary ecosystem analysis reports
heavy-tailed lengths and substantial intent-level redundancy across more than
40,000 public skills~\cite{claudeskillsanalysis}. Recent infrastructures
further support large-scale creation and orchestration
\cite{agentskillos,skillnet}, while \textsc{SkillsBench}~\cite{skillsbench},
SkillRet~\cite{skillret}, SRA~\cite{sra}, SkillGenBench~\cite{skillgenbench},
and SWE-Skills-Bench~\cite{sweskillsbench} evaluate skill utility, retrieval,
generation, and compatibility. Together, these studies establish skills as
durable procedural memory, although the complete skill package usually remains
the unit exposed to retrieval and evaluation.

\myparagraph{Skill organization, retrieval, and execution}
Structured skill representations provide a complementary foundation. AIP
models a skill as a schema-validated directed execution graph with typed I/O
edges~\cite{aipgraph}, while the Scheduling--Structural--Logical
representation separates skill-level scheduling, scene-level execution
structure, and action/resource evidence extracted from textual
skills~\cite{sslskills}. These approaches make individual skill artifacts more
machine-readable for execution, discovery, or governance. \method instead
uses source-grounded contract-bearing sections as cross-skill compression
units and couples them to reversible macro rewriting and task-time hydration.
Progressive disclosure reduces initial context by loading skill metadata before
full bodies~\cite{agentskills}. Graph-of-Skills~\cite{graphofskills} and
Group-of-Skills~\cite{groupofskills} 
then model dependencies, groups, conflicts, or specializations to retrieve
structurally coherent skill bundles. SkillRAE~\cite{skillrae} compiles reusable
subunits, and SkillLens~\cite{skilllens} adapts retrieval across policy,
strategy, procedure, and primitive levels. Post-retrieval execution-graph
systems construct task-time DAGs after multiple whole skills have been
selected~(e.g., \textsc{SkillDAG}~\cite{skilldag} and GRASP \cite{grasp}), whereas SkillNet~\cite{skillnet} connects
skills through an ecosystem-level ontology. SkillGraph~\cite{skillgraph}
represents skills as nodes in an evolving directed graph, retrieves ordered
skill subgraphs, and updates prerequisite, enhancement, and co-occurrence
relations from trajectory feedback. SkillOps~\cite{skillops} instead associates
each skill with a typed contract and maintains a hierarchical ecosystem graph
using utility, compatibility, risk, and validation signals. Related graph-based
systems use topology for
multi-step evidence retrieval, memory, and routing
\cite{tan2025paths_over_graph, whengraphsrag,graphplanner}. These methods improve which
skills an agent selects and how selected skills are orchestrated. Their graph
units nevertheless remain whole skills or task-time invocations rather than
recurring section subgraphs compressed into reversible executable macros.

\myparagraph{Skill compression and procedural memory}
LLMLingua and its extensions~\cite{llmlingua,longllmlingua,llmlingua2} compress
prompt sequences, while SkillReducer~\cite{skillreducer}, SkillEE~\cite{skillee},
and SKIM~\cite{skim} reduce skill context through token reduction, cost-aware
rewriting, or compact multi-resolution representations. Experience Compression
Spectrum~\cite{ecs} instead relates memories, skills, and rules as levels of
experience abstraction. Skill-Pro~\cite{skillpro},
ReasoningBank~\cite{reasoningbank}, MEM1~\cite{mem1}, and
MemGen~\cite{memgen} learn or organize reusable memory from agent trajectories;
recent procedural-memory management further studies how such knowledge should
be controlled, adapted, and evaluated over time~\cite{after}. These approaches
shorten prompts, rewrite individual skills, or abstract successful experience.
Cross-skill compression additionally requires repeated procedures to remain
distinguishable by their execution interfaces, dependencies, and verification
conditions.

\myparagraph{Graph summarization and grammar-based compression}
Frequent-subgraph methods discover recurring structure through
description-length search or canonical enumeration
\cite{subdue1994,gspan2002,gaston2004}. Graph summarization represents large
graphs through attribute-aware grouping~\cite{snap2008},
structure-preserving summaries~\cite{grass2010}, or MDL-selected vocabularies
\cite{vog2014}. Scalable and lossless methods further optimize summary
construction, query preservation, and storage
\cite{sweg2019,ssumm2020,slugger2022}. Grammar-based graph
compression~\cite{graphgrammar2018} provides reversible replacement rules and
supports reachability or regular-path queries over compressed representations,
while MoSSo~\cite{mosso2020} studies incremental lossless maintenance.
Accordingly, motif discovery, MDL selection, reversible grammars, compressed
querying, and incremental summaries are established graph techniques. \method
adapts them to procedural skill libraries by validating every occurrence
against typed boundary ports, dependency closure, verifier reachability, and
source provenance before replacement, then hydrating the compressed graph
under a task budget.

\newpage
\section{Prompts}
\label{appendix:prompt}

This section summarizes the model-assisted prompts used by \method. They
correspond to two operations in the main pipeline. Sec2Graph first analyzes
source-grounded sections and then normalizes their cross-section data flow;
PathHydrate converts each task into the structured anchors used for seed
retrieval. The former is performed offline and cached with the library, while
the latter is cached per task. All calls use temperature zero. This generation
limit is separate from the procedural-context budget: section-level extraction
uses a 512-token output limit, per-skill normalization uses 2,048 tokens, and
task anchoring uses the configured structured-output limit.
The remaining modules do not introduce hidden prompt steps. MotifZip,
procedural graph construction, constrained subgraph search, closure repair,
macro rendering, and ReZip are deterministic graph operations. Dense retrieval
uses embeddings over the task anchors and section contents. For reproducibility,
we separately report the thin benchmark interfaces that deliver $C_q$ to the
task executor and allow targeted source expansion. These interfaces do not
change the procedural graph or its contracts.

\noindent\textbf{Notation.}
$s$ denotes a source skill, $b$ a source-grounded section, $q$ a task query,
$p$ a benchmark profile, and $C_q$ the hydrated executable context. Variables
enclosed by braces are instantiated at runtime.

\phantomsection
\label{prompt:section-contract}
\myparagraph{Section role and contract extraction}
This prompt implements the model-assisted analysis described in
Section~\ref{sec:method:sec2graph}. It receives one candidate section after
source-preserving segmentation and returns the role and local contract fields
of its section node. The SkillsBench profile describes technical software-agent
skills; the \alfworld profile describes navigation and household manipulation.
Profile-specific demonstrations using the same schema are prepended as
user--assistant turns.

\begin{tcolorbox}[
  title=Section Role and Contract Extraction Prompt Template,
  title after break=Section Role and Contract Extraction Prompt Template (continued),
  colback=purplemem!7,
  colframe=purplemem,
  boxrule=0.6pt,
  arc=1mm,
  breakable,
  fonttitle=\bfseries]
\small
\textbf{System role.}
You analyze one section of a procedural skill document for an AI agent.
The domain is \texttt{\{DomainProfile\}}. Extract the section's structured
execution semantics and return one JSON object without prose or markdown.

\textbf{Execution roles.}
Choose one or more roles from the following set and list the most specific role
first:

\texttt{intent}: the overall goal of the skill;\quad
\texttt{trigger}: when the skill should be invoked;\quad
\texttt{input}: values or parameters consumed by the procedure;\quad
\texttt{precondition}: world or environment conditions required before
execution;\quad
\texttt{operation}: an executable action, function call, algorithm step, or
code block;\quad
\texttt{resource}: a required tool, API, library, file, credential, or
device;\quad
\texttt{failure}: an error case or recovery condition;\quad
\texttt{verifier}: an observable check of successful completion;\quad
\texttt{output}: the resulting artifact or post-condition.

\textbf{Extraction rules.}
Use concise noun phrases for input and output signatures. Preconditions must
block execution when false; effects describe state changes rather than the
action itself; verifier hooks must be observable strings, return values, or
conditions. Do not infer unsupported fields. Use an empty list when a field is
absent.

\textbf{Return format (strict JSON):}

\texttt{\{}\\
\texttt{\ \ "roles": ["..."],}\\
\texttt{\ \ "input\_signature": ["..."],}\\
\texttt{\ \ "output\_signature": ["..."],}\\
\texttt{\ \ "preconditions": ["..."],}\\
\texttt{\ \ "effects": ["..."],}\\
\texttt{\ \ "verifier\_hooks": ["..."]}\\
\texttt{\}}

\textbf{Final user turn.}\\
\texttt{Skill: \{SkillName\}}\\
\texttt{Section:}\\
\texttt{\{SectionText\}}
\end{tcolorbox}

For \textsc{SkillsBench}, \texttt{\{DomainProfile\}} specifies programming,
scientific computing, data analysis, security, machine learning, visualization,
or another engineering domain. For \alfworld, it specifies the text-based
household environment and its navigation, pickup, placement, cleaning, heating,
cooling, toggling, and slicing actions. For example, a cleaning instruction
that executes an action and checks the resulting observation is labeled with
both \texttt{operation} and \texttt{verifier}.

\phantomsection
\label{prompt:signature-canonicalization}
\myparagraph{Cross-section signature canonicalization}
Local extraction can assign different names to the same artifact. Sec2Graph
therefore performs a second offline pass over all signature-bearing sections of
one skill. The prompt normalizes producer and consumer signatures without
changing section boundaries, roles, source pointers, or skill membership. Its
output supports the typed dependency edges used by procedural subgraph
construction.

\begin{tcolorbox}[
  title=Shared Contract Vocabulary Prompt Template,
  title after break=Shared Contract Vocabulary Prompt Template (continued),
  colback=purplemem!7,
  colframe=purplemem,
  boxrule=0.6pt,
  arc=1mm,
  breakable,
  fonttitle=\bfseries]
\small
You are a data-flow analyst for AI software-agent skill documents. You are
given all numbered procedural sections of one skill. Build a shared artifact
vocabulary so that an artifact produced by one section has exactly the same
canonical name when consumed by another.

First identify every artifact or state passed between sections. Assign each a
single descriptive \texttt{snake\_case} noun and reuse that token everywhere.
For example, use \texttt{periodogram} rather than separate tokens such as
\texttt{periodogram\_object}, \texttt{power\_array}, or
\texttt{lomb\_scargle\_result} for the same artifact.

For every section, list the canonical tokens it consumes and produces,
together with guards that must already hold. External files, datasets,
parameters, and environment states may appear as inputs without an internal
producer. Sections with no procedural data flow receive empty lists. Do not
copy inconsistent seed signatures from the input; infer one shared vocabulary
from the section contents.

\textbf{Return format (strict JSON):}

\texttt{\{}\\
\texttt{\ \ "\{SectionIndex\}": \{}\\
\texttt{\ \ \ \ "in": ["..."],}\\
\texttt{\ \ \ \ "out": ["..."],}\\
\texttt{\ \ \ \ "guards": ["..."]}\\
\texttt{\ \ \}, ...}\\
\texttt{\}}

\texttt{Skill: \{SkillName\}}\\
\texttt{Sections:}\\
\texttt{[\{Index\}] roles=\{Roles\}}\\
\texttt{\{SectionText\}}\\
\texttt{...}
\end{tcolorbox}

\phantomsection
\label{prompt:skillsbench-anchoring}
\myparagraph{\textsc{SkillsBench} task anchoring}
At inference time, PathHydrate maps an artifact-oriented task brief and its
agent-visible metadata to the structured task object $z_q$ defined in
Section~\ref{sec:method:pathhydrate}. The prompt expresses deliverables using
the same signature vocabulary as Sec2Graph and decomposes the task into ordered
subgoals for section-level seed retrieval.

\begin{tcolorbox}[
  title=Technical Task Anchoring Prompt Template,
  title after break=Technical Task Anchoring Prompt Template (continued),
  colback=kgblue!7,
  colframe=kgblue,
  boxrule=0.6pt,
  arc=1mm,
  breakable,
  fonttitle=\bfseries]
\small
You analyze a programming or data task for a node-centric skill-retrieval
engine. The engine retrieves procedural section nodes and assembles an
executable subgraph. Re-express the task so that its deliverables and operations
match section output signatures and operation contents.

Return one JSON object with exactly these fields:
\texttt{goal}, a one-sentence objective;
\texttt{domain}, a short domain tag;
\texttt{target\_outputs}, two to six concrete artifacts written as
\texttt{snake\_case} signature tokens;
\texttt{capabilities}, three to eight concise verb--object operations;
\texttt{inputs}, the resources supplied by the task; and
\texttt{sub\_goals}, an ordered list of two to five steps.
Each subgoal contains \texttt{index}, \texttt{description},
\texttt{target\_output}, \texttt{retrieval\_query}, and local
\texttt{capabilities}. Use the task's nouns and do not invent unrelated steps.

\textbf{Return format (strict JSON):}

\texttt{\{}\\
\texttt{\ \ "goal": "...", "domain": "...",}\\
\texttt{\ \ "target\_outputs": ["..."],}\\
\texttt{\ \ "capabilities": ["..."],}\\
\texttt{\ \ "inputs": ["..."],}\\
\texttt{\ \ "sub\_goals": [\{}\\
\texttt{\ \ \ \ "index": 1, "description": "...",}\\
\texttt{\ \ \ \ "target\_output": "...",}\\
\texttt{\ \ \ \ "retrieval\_query": "...",}\\
\texttt{\ \ \ \ "capabilities": ["..."]}\\
\texttt{\ \ \}]}\\
\texttt{\}}

\texttt{TASK BRIEF: \{TaskBody\}}\\
\texttt{DOMAIN HINT: \{Domain\}}\\
\texttt{OUTPUT FILES: \{OutputFiles\}}
\end{tcolorbox}

The domain hint and output-file list are derived from the same task brief and
metadata exposed to the executor. Benchmark verifier source, assertions, and
outcomes are not inputs to this prompt; they are used only after execution to
score the produced artifact.

\phantomsection
\label{prompt:alfworld-anchoring}
\myparagraph{\alfworld task anchoring}
The embodied profile produces the same task object $z_q$, but makes implicit
navigation, pickup, state-change, and placement steps explicit. This is needed
because an \alfworld goal often names only the target state and destination,
whereas retrieval must cover the complete action chain.

\begin{tcolorbox}[
  title=Embodied Task Anchoring Prompt Template,
  title after break=Embodied Task Anchoring Prompt Template (continued),
  colback=kgblue!7,
  colframe=kgblue,
  boxrule=0.6pt,
  arc=1mm,
  breakable,
  fonttitle=\bfseries]
\small
You analyze an embodied household task for an action-skill retrieval engine.
The skill library contains the following action families:
navigation to a location or receptacle, object pickup, object placement,
cleaning at a sink basin, heating with an appliance, cooling with a fridge,
device operation, and slicing with a tool.

Construct the full executable action chain. Every object-manipulation task must
make explicit the steps needed to navigate to the object, pick it up, navigate
to the destination, and place it. Insert cleaning, heating, cooling, toggling,
or slicing when required by the goal, in the order in which the agent must
execute them.

Return the same JSON schema as the technical task-anchoring prompt:
\texttt{goal}, \texttt{domain}, \texttt{target\_outputs},
\texttt{capabilities}, \texttt{inputs}, and ordered
\texttt{sub\_goals}. Use two to five end-state tokens and three to six subgoals.
Phrase each \texttt{retrieval\_query} using action verbs that directly match the
corresponding skill family. Return strict JSON only.

\texttt{TASK BRIEF: \{TaskBody\}}\\
\texttt{DOMAIN HINT: embodied household navigation}\\
\texttt{VISIBLE OBJECTS OR RECEPTACLES: \{Inputs\}}
\end{tcolorbox}

Both task-anchoring profiles are parsed into the common representation $z_q$.
The raw query is retained alongside this abstraction, allowing the dual-level
seed fusion in PathHydrate to combine signature-oriented subgoal matching with
the original task wording.

\myparagraph{Benchmark execution interfaces}
The following templates are not additional graph-construction or retrieval
stages. They specify how the context produced by PathHydrate is exposed to each
benchmark executor. The task statement and hydrated section contents are
retained unchanged; the interface adds only the retrieval and expansion
controls needed to evaluate progressive hydration.

\myparagraphunderline{\textsc{SkillsBench} context delivery}
\phantomsection
\label{prompt:skillsbench-execution}
For \textsc{SkillsBench}, the hydrated plan is written to the instruction file
read by the OpenHands or command-line executor. The template leaves the
executor's base prompt unchanged and exposes the complete source library through
a search tool. The bounded search policy prevents the agent from spending its
execution budget on repeated retrieval after the required procedure is clear.

\begin{tcolorbox}[
  title=\textsc{SkillsBench} Execution-Context Interface,
  title after break=\textsc{SkillsBench} Execution-Context Interface (continued),
  colback=softgray,
  colframe=darkgray,
  boxrule=0.6pt,
  arc=1mm,
  breakable,
  fonttitle=\bfseries]
\small
Use the following \method execution plan to solve the task:

\texttt{\{HydratedContext\}}

A skill-search tool over the complete source library is available. Locate the
tool once. At the start, issue one or two short searches using the required
artifact, format, operation, or API, and inspect the one or two most relevant
source skills for exact parameters or steps.

Then commit to implementing the solution. Do not continue searching after the
required procedure is clear. Search again only when execution reaches a
specific obstacle that another skill can address. Prefer the shortest path to
producing and verifying the required output artifacts.

\texttt{Task: \{TaskBrief\}}\\
\texttt{Workspace and available files: \{EnvironmentState\}}
\end{tcolorbox}

\newpage
\phantomsection
\label{prompt:alfworld-execution}
\myparagraphunderline{\alfworld action and source-expansion interface}
For \alfworld, the executor receives the hydrated context together with the
current observation and admissible actions. It returns either one environment
action or one request for missing procedural detail. An expansion request is
appended to the task and recent interaction state, after which PathHydrate
returns an updated $C_q$ within the fixed expansion budget.

\begin{tcolorbox}[
  title=\alfworld Action and Expansion Interface,
  title after break=\alfworld Action and Expansion Interface (continued),
  colback=softgray,
  colframe=skillgreen,
  boxrule=0.6pt,
  arc=1mm,
  breakable,
  fonttitle=\bfseries]
\small
You are an \alfworld agent. On each turn, return exactly one JSON object.

\textbf{Action turn:}\\
\texttt{\{"thought": "<one short sentence>",}\\
\texttt{\ \ "action": "<one admissible environment action>"\}}

\textbf{Expansion turn:}\\
\texttt{\{"thought": "<one short sentence>",}\\
\texttt{\ \ "expand\_request": "<the missing procedural detail>"\}}

Use exactly one of \texttt{action} or \texttt{expand\_request}. Prefer an
environment action when the next move is clear. Request expansion only when the
current hydrated context lacks a concrete procedural detail. Do not invent
tools or commands. State-changing actions such as clean, heat, or cool do not
place an object; execute a separate placement action when the goal requires it.

\texttt{Episode: \{Episode\}}\\
\texttt{Goal: \{TaskQuery\}}\\
\texttt{Hydrated executable context: \{HydratedContext\}}\\
\texttt{Current observation: \{Observation\}}\\
\texttt{Admissible actions: \{AdmissibleActions\}}
\end{tcolorbox}

\end{document}